\documentclass[12pt]{article}

\usepackage{enumerate}
\usepackage{url} 

\newcommand{\blind}{0}

\usepackage{cite}
\usepackage{natbib}   
\usepackage{setspace}

\usepackage{amsmath,amssymb,amsfonts,amsthm}
\usepackage{bm}
\usepackage{hyperref}
\usepackage{soul, color}
\usepackage{float}
\usepackage{caption}
\usepackage{multirow}
\usepackage[linesnumbered,algoruled,boxed,lined]{algorithm2e}
\SetKwInput{KwInput}{Input}                
\SetKwInput{KwOutput}{Output}              
\SetKwInOut{Parameter}{Parameters}
\usepackage{comment}

\theoremstyle{plain}

\usepackage{algorithmic}
\usepackage{graphicx}
\usepackage{subfigure}
\usepackage{textcomp}
\usepackage{xcolor}
\def\BibTeX{{\rm B\kern-.05em{\sc i\kern-.025em b}\kern-.08em
    T\kern-.1667em\lower.7ex\hbox{E}\kern-.125emX}}

\usepackage{xspace}

\newtheorem{lemma}{Lemma}[section]

\newtheorem{theorem}{Theorem}

\newtheorem{proposition}[theorem]{Proposition}

\newcommand{\name}{\texttt{perTucker}\xspace}
\newcommand{\tsA}{{\bm{\mathcal A}}}
\newcommand{\tsB}{{\bm{\mathcal B}}}
\newcommand{\tsC}{{\bm{\mathcal C}}}

\newcommand{\tsX}{{\bm{\mathcal X}}}
\newcommand{\tsY}{{\bm{\mathcal Y}}}
\newcommand{\tsR}{{\bm{\mathcal R}}}

\newcommand{\matA}{{\bm{A}}}
\newcommand{\matB}{{\bm{B}}}

\newcommand{\matI}{{\bm{I}}}

\newcommand{\matU}{{\bm{U}}}

\newcommand{\matV}{{\bm{V}}}
\newcommand{\tmatV}{{\tilde{\bm{V}}}}

\newcommand{\setK}{{\mathcal{K}}}
\newcommand{\norm}[1]{\left\lVert#1\right\rVert}
\newcommand{\abs}[1]{\left\lvert#1\right\rvert}
\begin{document}

\def\spacingset#1{\renewcommand{\baselinestretch}%
{#1}\small\normalsize} \spacingset{1}

\onehalfspacing  

\if0\blind
{
  \title{\bf Personalized Tucker Decomposition: Modeling Commonality and Peculiarity on Tensor Data}
  \author{Jiuyun Hu$^1$, Naichen Shi$^2$, Raed Al Kontar$^2$, Hao Yan$^1$\\
    $^1$School of Computing and Augmented Intelligence\\
    Arizona State University\\
    $^2$Department of Industrial \& Operations Engineering\\
    University of Michigan, Ann Arbor\\
    }
  \maketitle
} \fi

\if1\blind
{
  \bigskip
  \bigskip
  \bigskip
  \begin{center}
    {\LARGE\bf Personalized Tucker Decomposition: Modeling Commonality and Peculiarity on Tensor Data}
\end{center} 
  \medskip
} \fi

\bigskip
\begin{abstract}
We propose personalized Tucker decomposition (\name) to address the limitations of traditional tensor decomposition methods in capturing heterogeneity across different datasets. \name decomposes tensor data into shared global components and personalized local components. We introduce a mode orthogonality assumption and develop a proximal gradient regularized block coordinate descent algorithm that is guaranteed to converge to a stationary point. By learning unique and common representations across datasets, we demonstrate \name's effectiveness in anomaly detection, client classification, and clustering through a simulation study and two case studies on solar flare detection and tonnage signal classification.
\end{abstract}

\noindent%
{\it Keywords:} Tucker decomposition, Personalization, Heterogeneous data
\vfill

\newpage


\section{Introduction} \label{sec:Introduction}

In recent years, tensor decomposition methods have grown rapidly, providing the ability to analyze and utilize high-dimensional data structures, which are essentially multi-dimensional arrays or matrices. These decompositions are powerful mathematical tools that facilitate the extraction of latent features and patterns from complex data, enabling efficient data representation, dimensionality reduction, compression, completion, noise removal, and prediction. Indeed, tensor decomposition has seen immense success across a wide variety of applications that include: natural image and video processing \citep{gatto2021tensor, momeni2022high}, health care systems \citep{ren2022blockchain, sandhu2018tdrm}, point cloud data \citep{yan2019structured, du2022tensor} and manufacturing \citep{zhen2023image, yan2014image}, amongst many others. 


Among the widely used techniques in this area, Tucker decomposition stands out as a prominent approach that has been successfully tested and deployed in various settings and applications \citep{kolda2009tensor, zubair2013tensor, li2020tensor}. The Tucker approach generalizes singular value decomposition (SVD) to higher-order tensors, providing a core tensor and a set of factor matrices that capture the interactions between dimensions \citep{tucker1966some}. The factor matrices represent the underlying patterns and structures in the data, while the core tensor captures the interaction between these patterns. By analyzing factor matrices and the core tensor, one can identify and extract meaningful features that can be used for further analysis, such as anomaly detection \citep{yan2014image} and process optimization \citep{yan2019structured}. 

Despite the efficacy of tensor decomposition methods, they assume that complex, heterogeneous data can be adequately represented by a single set of global factor matrices and a core tensor. This assumption may oversimplify the intrinsic disparities that exist when the datasets come from different sources, clients, or modalities, potentially compromising the accuracy of the resulting representations. In practice, nowadays, it is common to collect data across various edge devices, such as sensors and phones, which exhibit unique local data patterns due to various local conditions, system status, or data collection methodologies \citep{kontar2021internet}. 

Using a universal tensor decomposition strategy may not accurately capture these distinct data patterns, leading to suboptimal representations. An alternative strategy involves fitting a local tensor decomposition for the data source. However, this does not utilize the rich data available across sources and may excessively overfit the peculiarities of each dataset while neglecting the shared patterns or commonalities among the datasets. More importantly, both strategies overlook the opportunity to model heterogeneity across data sources and exploit this for improved downstream analysis, be it in prediction, clustering, classification, or anomaly detection.


For example, in the context of one of our case studies on tonnage signal monitoring processes, numerous sensors are employed to monitor the tonnage force at various locations within a production system. The data collected at each location features common patterns of normal signal variations and heterogeneous failure-related patterns. Here, the heterogeneous nature of the data and the presence of diverse failure patterns pose significant challenges for traditional tensor decomposition methods. Therefore, it is essential to develop a personalized tensor decomposition that can effectively capture and represent commonality and peculiarity across the data collected from each location. Consequently, these methods could reveal previously hidden patterns and relationships, allowing more effective data analysis, decision-making, and system optimization.





Inspired by a recent personalization technique for vector datasets coined as personalized principal component analysis (PCA) \citep{shi2022personalized}, we propose the personalized Tucker decomposition (\name) to decompose tensor data collected from different sources into shared global components and personalized components to capture the heterogeneity of the data. Global components model the common patterns shared across the datasets, while local components model the unique features of a specific dataset. At the heart of our approach is (i) a mode orthogonality constraint to distinguish global and local features and (ii) a proximal gradient-regularized block coordinate descent algorithm that operates within feasible regions of the constraint and can provably recover stationary solutions.

Using two case studies and simulated data, we highlight the ability of \name to benefit (i) anomaly detection as one can monitor changes in the local features to better (and faster) detect anomalies in data collected over time and (ii) classification \& clustering as operating on local features may yield better statistical power than the raw data since differences are more explicit when global features are removed.

The remainder of the paper is organized as follows. Sec. \ref{sec:Literature} reviews relevant literature on tensor decomposition methods. Sec. \ref{sec:Formulation} introduces \name, proposes an algorithm to estimate model parameters, proves convergence, and sheds light on potential applications of \name. Sec. \ref{sec:Simulation} uses simulated data to highlight the advantageous properties of our model. Two case studies on solar flare detection and tonnage signal classification are then presented in Sec. \ref{sec:CaseStudy}. Finally, Sec. \ref{sec:Conclusion} concludes the paper with a discussion about open problems.

We note that hereon we will use data source, client and modality interchangeably to index the origin from which each dataset was created. 

\section{Literature Review} \label{sec:Literature}

Various tensor decomposition methods have been proposed in the literature. Among them, Tucker \citep{tucker1966some} and CP decompositions \citep{hitchcock1927expression} have received the most attention. They have been applied to both supervised and unsupervised learning tasks. 

For unsupervised tasks, great emphasis was placed on anomaly detection and clustering. In anomaly detection, 
starting from the work of \citet{nomikos1994monitoring}, tensor-based detection has grown dramatically in the literature. Some examples include  \citet{li2011robust}, where a robust tensor subspace learning algorithm is used for online anomaly detection, and \citet{yan2014image}, which studied the relationship between Tucker decomposition, CP decomposition, multilinear principal component analysis, and tensor rank one decomposition and proposed monitoring statistics for each method. Interested readers are referred to \citet{fanaee2016tensor} for an overview of existing tensor-based methods for anomaly detection. 

In clustering, various methods have been proposed to improve the accuracy and efficiency of clustering algorithms. These methods include tensor-based subspace clustering \citep{fu2016tensor}, multi-view clustering \citep{zhang2023multi, li2023multi}, and multi-mode clustering \citep{he2022multi}. 

Within these areas, \citet{sun2019dynamic} developed a dynamic tensor clustering method based on CP tensor decomposition, which does not limit the tensor order and can be learned efficiently. \citet{wu2016general} proposed tensor spectral co-clustering based on a stochastic Markov process. This method works for general tensors of any order and can simultaneously cluster all dimensions. \citet{zhou2019tensor} proposed a tensor low-rank reconstruction technique (TLRR). The reconstruction consists of a low-rank dictionary recovery component and sparse noise. The dictionary is then used to obtain a similarity matrix to cluster the data. 

For supervised tasks, such as regression and classification, various tensor-based classification methods have been developed, including logistic tensor regression \citep{tan2013logistic}, support tensor machine \citep{hao2013linear}, and tensor Fisher discriminant analysis \citep{yan2005discriminant}. Furthermore, different forms of tensor regression have been proposed, depending on the dimensionality of the input and output variables. These include scalar-to-tensor regression \citep{zhou2013tensor}, tensor-to-scalar regression \citep{yan2019structured}, and tensor-to-tensor regression \citep{gahrooei2021multiple}. 
 
Here it is worth noting that a large body of literature has focused on separating noise from a low-rank background, starting from pioneering work on robust PCA \citep{candes2011robust}. Subsequently, this separation has been extended to anomaly detection by decomposing the data into three parts: background, anomaly, and noise. For example, \citet{yan2017anomaly} proposed a smooth sparse decomposition (SSD) by decomposing large-scale image data into a smooth background, a sparse anomaly, and noise and outperformed many other image-based detectors. Similar approaches can be found in crime monitoring \citep{zhao2022rapid}, public health surveillance \citep{dulal2022covid, zhao2020rapid}, and transfer learning applications \citep{li2022profile}. Unfortunately, such methods suffer from an inability to learn the data representation as they assume that the basis functions or data representation are known, which limits their ability to handle complex datasets. To mitigate this, recent methods have been proposed to learn the representation of background components, including Bayesian methods \citep{guo2022bayesian} and deep neural networks \citep{zhao2022deep}. Still, such approaches cannot simultaneously learn the basis functions for the background and anomaly components of the dataset.

Given the above literature, to the best of our knowledge, to date, there are no tensor decomposition methods capable of learning shared and common representations across different datasets. The closest work along this line is personalized PCA \texttt{perPCA}. \texttt{perPCA} introduces a novel technique to perform PCA on data originating from disparate sources or modalities that exhibit heterogeneous trends but also possess common characteristics \citep{shi2022personalized}. \texttt{perPCA} uses orthogonal global and local components to capture both shared and unique characteristics of each source. The paper offers a sufficient identifiability condition, theoretical guarantees, and competitive empirical results. Unfortunately, \texttt{perPCA} requires vectorization of datasets and cannot directly handle tensor data. The extension of \texttt{perPCA} to tensor data faces fundamental challenges due to the large degree of freedom and nonclosed-form solutions with tensor decompositions, the difficulty in defining tensor-based orthogonal constraints, and computational challenges involving high-order tensors. 

Our work aims to bring personalization to a tensor paradigm and address the challenges imposed to make that possible.

\section{Model Development} \label{sec:Formulation}
In this section, we first set the notation in Sec. \ref{subsec:Preliminary} followed by the motivation and formulation of \name in Sec. \ref{subsec:Motivation}. In Sec. \ref{subsec:Algorithm}, we propose an efficient algorithm to learn \name. Convergence, practical implementation, and potential applications are, respectively, highlighted in Sec. \ref{subsec:convergence}, Sec. \ref{subsec:Usage}, and Sec. \ref{subsec:practical}. We note that the proof of all propositions, lemmas, and theorems is deferred to the Appendix.

\subsection{Preliminary}\label{subsec:Preliminary}
A tensor can be regarded as a data structure with more than 2 dimensions, also known as modes in tensor analysis (see Fig. \ref{fig:Tensor_Example}). For example, in images, we use a vector of length $3$ to represent the RGB channel of the pixel. Thus, a picture can be represented by a 3-dimensional tensor with dimensions height$\times$width$\times$RGB. If we have multiple pictures of the same dimension, a dataset can be represented by a 4-dimensional tensor.

\begin{figure}[t]
    \centering
    \includegraphics[width=0.6\linewidth]{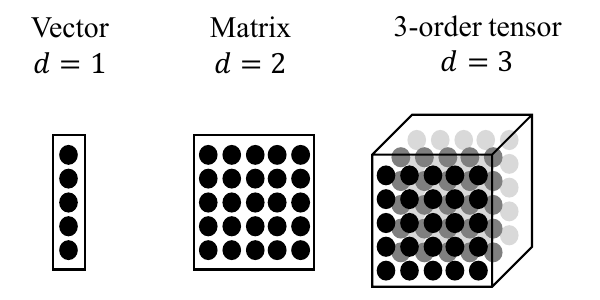}
    \caption{Example of Vector, Matrix and Tensor Data}
    \label{fig:Tensor_Example}
\end{figure}

\paragraph{Notation} Throughout this paper, real numbers are denoted by letters, e.g., $N$, $i$; vectors by bold lowercase letters, e.g., $\bm c$; matrices by bold uppercase letters, e.g., $\bm U$; sets by script letters, e.g., $\mathcal K$; and tensors by bold script letters, e.g., $\bm{\mathcal X}$. 

\paragraph{Mode-$k$ product and tensor unfolding} We briefly review the notion of a Tucker tensor product. A $K$-mode tensor is represented by $\bm{\mathcal{X}}\in\mathbb{R}^{I_{1}\times\cdots\times I_{K} }$, where $I_{k}$ denotes the mode-$k$ dimension of $\tsX$ for $k=1,\cdots, K$.  We use $\tsX[i_1,i_2,\cdots,i_K]$ to denote the $(i_1,i_2,\cdots,i_K)$-th entry of $\tsX$. The mode-$k$ product of a tensor $\tsX$ with a matrix $\bm{V}\in\mathbb{R}^{J_{k}\times I_{k}}$
produces a tensor defined by $(\bm{\mathcal{X}}\times_{k}\bm{V})[i_{1},\cdots,i_{k-1},j_{k},i_{k+1},\cdots,i_{K}]=\sum_{i_{k}}\tsX[i_{1},\cdots,i_{k},\cdots,i_{N}]V[j_{k},i_{k}]$. For a tensor $\bm{\mathcal{X}}$
and a specific mode $k$, we use the subscript with parenthesis $\tsX_{(k)}\in \mathbb{R}^{I_k\times\prod_{q=1,q\neq k}^K I_q}$ to denote the unfolding of $\tsX$ with respect to dimension $k$, $\tsX_{(k)}[i_k,j]=\tsX[i_1,i_2,\cdots,i_K]$ where $j=1+\sum_{q=1,q\neq k}^K (i_q-1)J_q$ and $J_q=\prod_{m=1,m\neq k}^{q-1}I_m$. The columns of the $k$-mode unfolding $\tsX_{(k)}$ are the n-mode vectors of $\tsX$.

\paragraph{Tucker decomposition}
Tucker decomposition \citep{tucker1966some} 
 decomposes a tensor into a core tensor multiplied by a factor matrix along each mode, $\tsX\approx\tsC\times_{1}\matU_{1}\times_{2}\matU_{2}\cdots\times_{K}\matU_{K}$,
where $\matU_{k}$ is an orthonormal $J_{k}\times I_{k}$ factor matrix typically with $J_k<I_k$. $\matU_k$ can be regarded as a principal component in mode-$k$. 

Tucker decomposition has an equivalent formulation in terms of the unfolded tensor, that is, $\tsX_{(k)}=\matU_k\tsC_{(k)}(\matU_{K}\bigotimes \cdots \bigotimes \matU_{k+1} \bigotimes\matU_{k-1} \cdots\bigotimes\matU_{1})^\top$. Here, $\bigotimes$ is the Kronecker product.

\paragraph{Tensor inner product} The inner product of two tensors of the same shape $\tsA, \tsB \in \mathbb{R}^{I_{1}\times\cdots\times I_{K}}$, is defined
$$\left \langle \tsA, \tsB \right \rangle = \sum_{i_{1},\cdots,i_{K}}\matA[i_{1},\cdots,i_{k},\cdots,i_{K}] \matB[i_{1},\cdots,i_{k},\cdots,i_{K}].$$
Then the Frobenius norm of a tensor $\tsA$ can be defined as $\norm{\tsA}_{F}^{2} = \left \langle \tsA, \tsA \right \rangle$, which is the sum of squares of all elements.


\subsection{Motivation \& formulation}\label{subsec:Motivation}
Suppose we have tensor data from $N$ sources. Each source has tensor data of order $K$. We use $\tsY_n$ to denote the data from source $n$, and assume that $\tsY_n$ has dimensions $I_1 \times I_2\times \ldots \times I_K \times s_n$. Here, all dimensions across sources have the same length except for the last one. In particular, in practical applications, the last dimension $s_n$ denotes the number of samples from the source $n$, which often differs between sources. 

Our approach relies on defining global and local components to model commonality and heterogeneity across different sources. To do so, we let the global components consist of shared global factor matrices $\matU_{G,1},\ldots, \matU_{G,K}$ and individual global core tensors $\bm{\mathcal C}_{G,1},\ldots,\bm{\mathcal C}_{G,N}$ for each source. The local components consist of individual core tensors $\bm{\mathcal C}_{L,1},\ldots,\bm{\mathcal C}_{L,N}$ and individual local factor matrices $\bm{V}_{n,1},\ldots,\bm{V}_{n,K}$. As such, the reconstructions of the global and local components for source $n$ are $\bm{\mathcal C}_{G,n}\times_1\matU_{G,1}\ldots\times_K\matU_{G,K}$ and $\bm{\mathcal C}_{L,n}\times_1\bm{V}_{n,1}\ldots\times_K\bm{V}_{n,K}$, respectively. 

Based on the above definitions, we assume our data-generating process to be
\begin{equation}
\bm{\mathcal Y}_n = 
\underbrace{\bm{\mathcal C}_{G,n}\times_1\matU_{G,1}\ldots\times_K\matU_{G,K}}_{\text{global}} + \underbrace{\bm{\mathcal C}_{L,n}\times_1\bm{V}_{n,1}\ldots\times_K\bm{V}_{n,K}}_{\text{local} \, i}+ \bm{\mathcal{E}}_n,
\label{eq: perTuckerModel}
\end{equation}
where $\bm{\mathcal{E}}_n$ are tensors that represent additive noise. 

Since global and local components should convey different information, they need to be distinguished so that each part can vary independently of each other. To do so, we require the orthogonality of the global and local tensors. Specifically, we assume that:

$$\left \langle \bm{\mathcal Y}_{G,n}, \bm{\mathcal Y}_{L,n} \right \rangle = 0, \quad \forall  \bm{\mathcal C}_{G,n}, \bm{\mathcal C}_{L,n},$$

\noindent where $\bm{\mathcal Y}_{G,n} = \bm{\mathcal C}_{G,n}\times_1\matU_{G,1}\ldots\times_K\matU_{G,K}$ 
and $\bm{\mathcal Y}_{L,n} = \bm{\mathcal C}_{L,n}\times_1\bm{V}_{n,1}\ldots\times_K\bm{V}_{n,K}$. Interestingly, it turns out that this condition is equivalent to having the global and local factor matrices orthogonal in at least one dimension, as stated in Proposition \ref{prop:orthogonal}. 
\begin{proposition}\label{prop:orthogonal} For each $n=1,\ldots,N$, the following two conditions are equivalent.
\begin{itemize}
    \item $\left \langle \bm{\mathcal Y}_{G,n}, \bm{\mathcal Y}_{L,n} \right \rangle = 0, \quad \forall  \bm{\mathcal C}_{G,n}, \bm{\mathcal C}_{L,n}$.
    \item There exists a mode $k\in\{1,\ldots,K\}$, where $\matU_{G,k}^\top\bm{V}_{n,k}=0$. 
\end{itemize}
\end{proposition}

Given Proposition \ref{prop:orthogonal}, we require local factor matrices to be orthogonal to global factor matrices for all sources in at least one mode. We define the set of such orthogonal modes by $\mathcal K$, $|\mathcal K|\ge 1$.
Then our objective is to minimize the reconstruction loss of the data across all $N$ sources. This is written as
{\small \begin{align}\label{eq:model}
    \min_{\{\bm{\mathcal C}_{G,n}\}, \{\matU_{G,k}\}, \{\bm{\mathcal C}_{L,n}\}, \{\bm{V}_{n,k}\}} &\sum_{n=1}^N \|\bm{\mathcal Y}_n - \bm{\mathcal C}_{G,n}\times_1 \matU_{G,1}\ldots \times_K \matU_{G,K}- \bm{\mathcal C}_{L,n}\times_1 \bm{V}_{n,1}  \ldots \times_K \bm{V}_{n,K} \|_F^2\\
    &s.t.~ \matU_{G,k}^\top\matU_{G,k}=I, \bm{V}_{n,k}^\top \bm{V}_{n,k} = I, n=1,\ldots,N, k=1,\ldots,K\notag\\
    & \matU_{G,k}^\top\bm{V}_{n,k}=0, n=1,\ldots,N, k\in\mathcal K.\notag
\end{align}}

We assume that the dimension of the global core tensor for all sources is $\tsC_{G,n}\in\mathbb R^{g_1\times\cdots\times g_K}$, and the dimension of the local core tensor for source $n$ is $\tsC_{L,n}\in\mathbb R^{l_{n,1}\times\cdots\times l_{n,K}}$. This also defines the dimension of the global and local factor matrices.


\subsection{Personalized Tucker algorithm}\label{subsec:Algorithm}

A natural algorithm to solve the objective in \eqref{eq:model} is block coordinate descent (BCD), where we iteratively optimize each variable. A general framework for BCD in our context is outlined in Algorithm \ref{alg:Pseudo_perTucker_BCD}.

\begin{algorithm}[h]
\caption{Pseudo Code of the Algorithm}\label{alg:Pseudo_perTucker_BCD}\footnotesize
\KwData{$\bm{\mathcal Y}_n$, $n=1,\ldots,N$}
\KwOutput{$\{ \bm{\mathcal C}_{G,n}\}, \{\matU_{G,k}\}, \{\bm{\mathcal C}_{L,n}\}, \{\bm{V}_{n,k}\}$}
\textbf{Initialization}: 
$\{ \bm{\mathcal C}_{G,n}\}, \{\matU_{G,k}\}, \{\bm{\mathcal C}_{L,n}\}, \{\bm{V}_{n,k}\}$\\ 
\For{iterations}{
    \For{$k=1,\ldots,K$}{
    Update global factor matrices $\{\matU_{G,k}\}$\\
        \For{$n=1,\ldots,N$}{
        Update global core tensors $\{ \bm{\mathcal C}_{G,n}\}$\\
        Update local factor matrices $\{\bm{V}_{n,k}\}$.\\ 
        Update core tensors $\{\bm{\mathcal C}_{L,n}\}$
        }
    }
}
\textbf{Return}: $\{ \bm{\mathcal C}_{G,n}\}, \{\matU_{G,k}\}, \{\bm{\mathcal C}_{L,n}\}, \{\bm{V}_{n,k}\}$
\end{algorithm}

In the rest of Sec. \ref{subsec:Algorithm}, we explain the update steps in Algorithm \ref{alg:Pseudo_perTucker_BCD} in detail. We start with the update of global and local core tensors in Sec. \ref{subsubsec:update_core} since the closed-form solution is a direct projection similar to the traditional Tucker decomposition owing to the orthogonality between global and local components. Then the solution consistently holds in the update of global and local factor matrices we introduced in Sec. \ref{subsubsec:update_global} and Sec. \ref{subsubsec:update_local}. This simplifies the update of the factor matrices without the core tensors. Despite the challenges that pertain to the two distinct decomposition components within \eqref{eq:model} and their orthogonality, \textit{a key result is that all updates can be done in closed form owing to the nice properties of the orthogonality constraint imposed on the model}.



\subsubsection{Update global and local core tensors}\label{subsubsec:update_core}
In this section, Proposition \ref{prop:Core} provides the closed-form solution to update global and local core tensors, given the global and local factor matrices. The closed-form solution is the direct projection of the data to the global or local factor matrices. When updating the global and local factor matrices, we assume that the global and local core tensors are always the optimal solution. This simplifies the formula to update the global and local factor matrices by removing the core tensors from the optimization problem.
\begin{proposition}\label{prop:Core}
(Closed-form solutions to the core tensor) If $\abs{\setK}\ge 1$, when the global factor matrices $\{\matU_{G,k}\}$ and the local factor matrices $\{\bm{V}_{n,k}\}$ are given, the global core tensors $\tsC^{\star}_{G,n}$ that minimize \eqref{eq:model} satisfy
$$
\tsC^{\star}_{G,n}=\bm{\mathcal Y}_n \times_1 \matU_{G,1}^\top\ldots \times_K \matU_{G,K}^\top,
$$
and the local core tensors $\tsC^{\star}_{L,n}$ that minimize \eqref{eq:model} satisfy
$$
\tsC^{\star}_{L,n}=\bm{\mathcal Y}_n \times_1 \bm{V}_{n,1}^\top \ldots \times_K \bm{V}_{n,K}^\top.
$$
\end{proposition}
The closed-form solutions presented in Proposition \ref{prop:Core} takes advantage of the orthogonality between the two components. As a result, the cross-term is canceled, making the computation of the core tensors for both components efficient and straightforward. In the following sections, $\tsC^{\star}_{G,n}$ and $\tsC^{\star}_{L,n}$ are used to denote optimized global and local core tensors.

\subsubsection{Update global and local factor matrices} \label{subsubsec:update_factor}

In this section, we will discuss the closed-form solutions to update the global and local factor matrices. For the simplicity of notation, we define the global residual tensor and local residual tensor from each source $n=1\ldots,N$ as:
$$
\tsR_{G,n} = \bm{\mathcal Y}_n - \tsC^{\star}_{L,n}\times_1 \bm{V}_{n,1}\ldots\times_K \bm{V}_{n,K},
$$
$$
\tsR_{L,n} = \bm{\mathcal Y}_n - \tsC^{\star}_{G,n}\times_1 \matU_{G,1}\ldots\times_K \matU_{G,K}.
$$
The global residual is the reconstruction error from local components and the local residual is the reconstruction error from global components. Therefore, global reconstruction tends to model the global residual, and local reconstruction tends to model the local residual.
\paragraph{Proximal update}
In practice, when updating the global and local factor matrices, we can incorporate a proximal term into the optimization problem to regulate the update of the factor matrices \citep{shen2022smooth}. More specifically, we can define $\varrho$ as the subspace difference between the subspaces expanded by the current factor matrix $\matU_t$ and the target factor matrix $\matU$ to be optimized,
\begin{equation}
\label{eqn:rhodef}
    \varrho (\matU, \matU_t)= \|\bm{UU}^\top-\matU_t\matU_t^\top\|_F^2.
\end{equation}

The proximal penalty term is defined as the subspace difference times some parameter $\rho$. The proximal gradient algorithm can stabilize the update of factor matrices by regularizing the subspace change. Since the reconstruction of global and local components are to minimize the reconstruction error, we can write the optimization problem to solve for the global factor matrix in mode $k$ at iteration $t$ as
    \begin{equation}    \label{eq:global-opt} \small \matU_{G,k,t+1}=\arg\min_{\matU_{G,k}} \sum_{n=1}^N \norm{\tsR_{G,n} - \tsC^{\star}_{G,n} \times_1 \matU_{G,1}\ldots\times_K \matU_{G,K}}_F^2 +\rho \|\matU_{G,k}\matU_{G,k}^\top-\matU_{G,k,t}{\matU_{G,k,t}^\top}\|_F^2,
    \end{equation}
and the optimization problem to solve for the local factor matrix of source $n$ is that when $k\in\mathcal K$,
    \begin{equation}\label{eq:local_opt_insetk}
    \matV_{n,k,t+1}=\arg\min_{\bm{V}_{n,k}\perp \matU_{G,k}} \norm{\tsR_{L,n} - \tsC^{\star}_{L,n} \times_1 \bm{V}_{n,1}\ldots\times_K \bm{V}_{n,K}}_F^2 +\rho \norm{\bm{V}_{n,k}\bm{V}_{n,k}^\top-\bm{V}_{n,k,t}{\bm{V}_{n,k,t}^\top}}_F^2,
    \end{equation}
and when $k\not\in\mathcal K$,
\begin{equation}\label{eq:local_opt_notinsetk}
    \matV_{n,k,t+1}=\arg\min_{\bm{V}_{n,k}} \norm{\tsR_{L,n} - \tsC^{\star}_{L,n} \times_1 \bm{V}_{n,1}\ldots\times_K \bm{V}_{n,K}}_F^2 +\rho \norm{\bm{V}_{n,k}\bm{V}_{n,k}^\top-\bm{V}_{n,k,t}{\bm{V}_{n,k,t}^\top}}_F^2,
    \end{equation}
where $\matU_{G,k,t}$ and $\matV_{n,k,t}$ represents the global and local factor matrices for source $n$, mode $k$ and iteration $t$; $\matU_{G,k,t+1}$ and $\matV_{n,k,t+1}$ are the corresponding updated global and local factor matrices. The objectives of \eqref{eq:local_opt_insetk} and \eqref{eq:local_opt_notinsetk} are the same. They consist of a Frobenius norm of the fitting error and a regularization on the change of subspace. The difference is that in \eqref{eq:local_opt_insetk}, we explicitly require $\matV_{n,k,t+1}$ to be orthogonal to $\matU_{G,k,t+1}$, while in \eqref{eq:local_opt_notinsetk} we do not add constraints on $\matV_{n,k,t+1}$. 

Though the optimization problems \eqref{eq:global-opt} to \eqref{eq:local_opt_notinsetk} seem complicated, it turns out we can obtain closed-form solutions. To achieve this, we first transform the minimization problem into a maximization problem and remove the core tensors in the optimization by Lemma \ref{lemma:subspace-equiv} and Lemma \ref{lemma:min-to-max}.

\begin{lemma}\label{lemma:subspace-equiv}
    For any orthonormal factor matrices $\matU$ and $\matU_t$, the subspace error between $\matU$ and $\matU_t$ defined in \eqref{eqn:rhodef} can be formulated as,
    \begin{equation}\label{eq:subspace-equiv}
   \varrho (\matU, \matU_t)  =2c-2Tr\left[\matU^\top\matU_t\matU_t^\top \matU\right],
    \end{equation}
    where $c$ is the number of rows in $\bm U$.
\end{lemma}
Lemma \ref{lemma:subspace-equiv} shows that the subspace error is differentiable with respect to $\matU$. This property is useful when we design the update rules for the BCD algorithms and the evaluation metrics. Furthermore, lemma \ref{lemma:subspace-equiv} put a negative sign in the matrix trace term that can transform the minimization problem into a maximization problem. 

Before deriving the solutions to \eqref{eq:global-opt} to \eqref{eq:local_opt_notinsetk}, we introduce the following lemma that significantly simplifies our objective.
 

\begin{lemma}\label{lemma:min-to-max}
For each $n=1,\ldots,N$ and $k=1,\ldots,K$, we have,
\begin{equation}\small\label{eq:global-equiv} 
    \sum_{n=1}^N \|\tsR_{G,n} - \tsC^{\star}_{G,n} \times_1 \matU_{G,1}\ldots\times_K \matU_{G,K} \|_F^2=-\sum_{n=1}^N \|\tsR_{G,n}\times_1 \matU_{G,1}^\top\ldots\times_K \matU_{G,K}^\top \|_F^2+\norm{\tsR_{G,n}}_F^2,
\end{equation}
\begin{equation}\label{eq:local-equiv}
    \|\tsR_{L,n} - \tsC^{\star}_{L,n} \times_1 \bm{V}_{n,1}\ldots \times_K \bm{V}_{n,K} \|_F^2 = -\|\tsR_{L,n}\times_1 \bm{V}_{n,1}^\top\ldots\times_K \bm{V}_{n,K}^\top\|_F^2+\norm{\tsR_{L,n}}_F^2.
\end{equation}

\end{lemma}
Lemma \ref{lemma:min-to-max} bears two fundamental meanings in the derivation of the closed-form solution to update the global and local factor matrices. First, it also puts a negative sign in the term with the factor matrices, which can transform the minimization problem \eqref{eq:global-opt} to \eqref{eq:local_opt_notinsetk} into a maximization problem. Second, by plugging in the closed-form solution of global and local core tensors in Proposition \ref{prop:Core}, it simplifies the optimization problem by reducing the number of decision variables. 

\paragraph{Update global factors}\label{subsubsec:update_global}
With all the prerequisites, we are now ready to present the closed-form solution of the sub-problem \eqref{eq:global-opt} in updating the global factor matrix $\matU_{G,k}$ in a specific mode $k$ in Proposition \ref{prop:Global_prox}


\begin{proposition}\label{prop:Global_prox}
We use $\bm{W}_{G,n}$ to denote $\bm{W}_{G,n}=\left(\tsR_{G,n}\right)_{(k)}(\bigotimes_{q\neq k}\matU_{G,q}^\top)^\top$, where $\bigotimes_{q\neq k}$ is the Kronecker product in reverse order of the factor matrices except $k$th factor matrix. If $\matU_{G,k,t+1}$ is the optimal solution to
\eqref{eq:global-opt}, the columns of $\matU_{G,k,t+1}$
    are the unit eigenvectors of the matrix $\sum_{n=1}^N \bm{W}_{G,n} \bm{W}_{G,n}^\top +2\rho \matU_{G,k,t} {\matU_{G,k,t}^\top}$ corresponding to the largest $g_k$ eigenvalues.
\end{proposition}

Proposition \ref{prop:Global_prox} shows that with proximal regularization, global components can be updated efficiently through singular value decomposition. In practice, we can use the equivalent form of $\bm{W}_{G,n}=(\tsR_{G,n}\times_1 \matU_{G,1} \cdots \times_{k-1} \matU_{G,k-1} \times_{k+1} \matU_{G,k+1}\cdots \times_K \matU_{G,K})_{(k)}$ to improve the efficiency of the computation.
Note that when updating the global components, we do not impose the orthogonality of the global and local components. The reason is that enforcing the orthogonality between the global components to each of the local components is too restrictive and may leave no feasible space for updating if the number of sources is large. Therefore, we will update the global components freely and enforce the local component to be orthogonal to the global components.

\paragraph{Update local factors}\label{subsubsec:update_local}

We provide the closed-form solution to the sub-problem \eqref{eq:local_opt_insetk} and \eqref{eq:local_opt_notinsetk} to update the local factor matrices with or without the orthogonal constraint in  Proposition \ref{prop:Local_prox}. The component optimized in Proposition \ref{prop:Local_prox} is the local factor matrix $\matV_{n,k}$ with a specific source $n$ and mode $k$. We denote the current local factor matrices at iteration $t$ by $\bm{V}_{n,k,t}$.

\begin{proposition}
\label{prop:Local_prox}
Problem \eqref{eq:local_opt_insetk} and \eqref{eq:local_opt_notinsetk} have closed-form solutions.
We denote $\bm{W}_{L,n}$ as $\bm{W}_{L,n}=\left(\tsR_{L,n}\right)_{(k)}(\bigotimes_{q\neq k}\bm{V}_{n,q}^\top)^\top$.
Then,    \begin{enumerate}
        \item if $k\not\in \mathcal K$, the updated columns of the local factor matrix $\bm{V}_{n,k,t+1}$ is the unit eigenvectors of $\bm{W}_{L,n} \bm{W}_{L,n}^\top + 2\rho \bm{V}_{n,k,t}{\bm{V}_{n,k,t}^\top}$ corresponding to top $l_{n,k}$ eigenvalues.

    \item if $k\in\mathcal K$, the update of the local factor matrix $\bm{V}_{n,k,t+1}$ is as follows.\\
    Denote $\bm{S}'=(I-\matU_{G,k}\matU_{G,k}^\top)[\bm{W}_{L,n} \bm{W}_{L,n}^\top + 2\rho \bm{V}_{n,k,t}{\bm{V}_{n,k,t}^\top}](I-\matU_{G,k}\matU_{G,k}^\top)$. The columns of the local factor matrix $\bm{V}_{n,k,t+1}$ are the eigenvectors of $\bm{S}'$ corresponding to top $l_{n,k}$ eigenvalues.    
    \end{enumerate}
\end{proposition}

The proof of Proposition \ref{prop:Global_prox} and Proposition \ref{prop:Local_prox} is shown in Appendix \ref{App:proof_global_prox} and Appendix \ref{App:proof_local_prox}. Having completed all the steps required to update each component of the algorithm, we present the complete algorithm in Algorithm \ref{alg:perTucker_BCD}. In Algorithm \ref{alg:perTucker_BCD}, we use the subscript $t$ to denote the current iteration index for the global and local factor matrices. Despite the orthogonality constraint between the local and global components, each update step in Algorithm \ref{alg:perTucker_BCD} can be efficiently implemented via a closed-form solution. 

\begin{algorithm}[h]
\footnotesize
\caption{BCD algorithm to solver Personalized Tucker}\label{alg:perTucker_BCD}
\KwData{$\bm{\mathcal Y}_n$, $n=1,\ldots,N$}
\KwInput{Global dimensions $g_k$, Local dimensions $l_{n,k}$, Orthogonal dimension set $\mathcal K$, $\rho$}
\KwOutput{$\{ \bm{\mathcal C}_{G,n}\}, \{\matU_{G,k}\}, \{\bm{\mathcal C}_{L,n}\}, \{\bm{V}_{n,k}\}$}
\textbf{Initialization}: Randomly initialize or PCA initialize $\{ \bm{\mathcal C}_{G,n}\}, \{\matU_{G,k,0}\}, \{\bm{\mathcal C}_{L,n}\}, \{\bm{V}_{n,k,0}\}$.\\ 
\For{iterations $t=0,\cdots,T-1$}{
    \For{$k=1,\ldots,K$}{
        Set $\tsR_{G,n} = \bm{\mathcal Y}_n - \bm{\mathcal C}^{\star}_{L,n}\times_1 \bm{V}_{n,1,t+1}\ldots \times_{k-1}\bm{V}_{n,k-1,t+1} \times_{k}\bm{V}_{n,k,t} \ldots \times_K \bm{V}_{n,K}$, $n=1,\ldots,N$\\
        Compute $\bm{W}_{G,n}=\left(\tsR_{G,n}\right)_{(k)}(\matU_{G,K,t}^\top \bigotimes \ldots \bigotimes \matU_{G,k+1,t}^\top\bigotimes \matU_{G,k-1,t+1}^\top \bigotimes \ldots \bigotimes \matU_{G,1,t+1}^\top)^\top$, $n=1,\ldots,N$\\    
        Update $\matU_{G,k,t+1}$ to be the eigenvectors of  $\sum_{n=1}^N \bm{W}_{G,n} \bm{W}_{G,n}^\top +2\rho \matU_{G,k,t} {\matU_{G,k,t}^\top}$ corresponding to the largest $g_k$ eigenvalues.  \\
        
        \For{$n=1,\ldots,N$}{
        Update $\bm{\mathcal C}^{\star}_{G,n}=\bm{\mathcal Y}_n \times_1 \matU_{G,1,t+1}^\top\ldots \times_k \matU_{G,k,t+1}^\top \times_{k+1} \matU_{G,k+1,t}^\top \ldots \times_K \matU_{G,K,t}^\top$\\
        Set $\tsR_{L,n} = \bm{\mathcal Y}_n - \bm{\mathcal C}^{\star}_{G,n}\times_1 \matU_{G,1,t+1}\ldots \times_k \matU_{G,k,t+1} \times_{k+1} \matU_{G,k+1,t} \ldots \times_K \matU_{G,K,t}$\\
        Let $\bm{W}_{L,n}=\left(\tsR_{L,n}\right)_{(k)}(\matV_{n,K,t}^\top \bigotimes \ldots \bigotimes \matV_{n,k+1,t}^\top\bigotimes \matV_{n,k-1,t+1}^\top \bigotimes \ldots \bigotimes \matV_{n,1,t+1}^\top)^\top$\\
        \uIf{$k\in\mathcal K$}{
            Let $\bm{S}'=(\matI-\matU_{G,k,t+1}\matU_{G,k,t+1}^\top)[\bm{W}_{L,n} \bm{W}_{L,n}^\top + 2\rho \bm{V}_{n,k,t}{\bm{V}_{n,k,t}^\top}](\matI-\matU_{G,k,t+1}\matU_{G,k,t+1}^\top)$\\
            Update $\bm{V}_{n,k,t+1}$ to be the eigenvectors of $\bm{S}'$ corresponding to the largest $l_{n,k}$ eigenvalues.
        } \ElseIf{$k\not\in\mathcal K$}{
            Update $\bm{V}_{n,k,t+1}$ to be the eigenvectors of $\bm{W}_{L,n} \bm{W}_{L,n}^\top + 2\rho \bm{V}_{n,k,t}{\bm{V}_{n,k,t}^\top}$ corresponding to the largest $l_{n,k}$ eigenvalues.
        }
        Update $\bm{\mathcal C}^{\star}_{L,n}=\bm{\mathcal Y}_n \times_1 \bm{V}_{n,1,t+1}^\top \ldots \times_k \bm{V}_{n,k,t+1}^\top \times_{k+1} \bm{V}_{n,k+1,t}^\top \times_K \bm{V}_{n,K,t}^\top$
    }
    }

}
\textbf{Return}: $\{ \bm{\mathcal C}^{\star}_{G,n}\}, \{\matU_{G,k,T}\}, \{\bm{\mathcal C}^{\star}_{L,n}\}, \{\bm{V}_{n,k,T}\}$
\end{algorithm}

\subsection{Convergence analysis of Algorithm \ref{alg:perTucker_BCD}}\label{subsec:convergence}
In this section, we provide the the convergence analysis of Algorithm \ref{alg:perTucker_BCD}. 
The special update rule in Proposition \ref{prop:Local_prox} brings challenges to the convergence analysis. In Algorithm \ref{alg:perTucker_BCD}, the update of local factors $\matV_{n,k}$'s is different from the standard Tucker decomposition update, as $\matV_{n,k}$ is required to be orthogonal to $\matU_{G,k}$. As a result, updating the local factors does not necessarily decrease the objective value in \eqref{eq:model}. Thus, Algorithm \ref{alg:perTucker_BCD} is not a strictly descent algorithm. Despite such subtleties, we can show that, when the proximal parameter $\rho$ is not too small, our algorithm can converge into stationary solutions. 

We will present our theorem on global convergence in the following theorem. Recall that we use $\matU_{G,k,t}$ to denote the $k$-th global factor $\matU_{G,k}$ after iteration $t$, and $\matV_{n,k,t}$ to denote the $k$-th local factor of source $n$ after iteration $t$.
\begin{theorem}
\label{thm:convergence}
If $|\setK|\ge 2$ and there exists a constant $B>0$ such that $\norm{\tsY_{n}}_F\le B$ for each $n$, when we choose $\rho=O(B^2)$, then Algorithm \ref{alg:perTucker_BCD} will converge to stationary points where
\begin{equation}
\begin{aligned}
&\min_{t=1,\cdots,T}\sum_{k=1}^{K}\norm{\matU_{G,k,t+1}\matU_{G,k,t+1}^\top-\matU_{G,k,t}\matU_{G,k,t}^\top}_F^2= O\left(\frac{1}{T}\right),
\end{aligned} 
\end{equation}
and
\begin{equation}
\begin{aligned}
&\min_{t=1,\cdots,T}\sum_{n=1}^N\sum_{k=1}^K\norm{\matV_{n,k,t+1}\matV_{n,k,t+1}^\top-\matV_{n,k,t}\matV_{n,k,t}^\top}_F^2= O\left(\frac{1}{\sqrt{T}}\right).
\end{aligned} 
\end{equation}
\end{theorem}

Theorem \ref{thm:convergence} provides many key insights. First, it shows that the subspaces spanned by the column vectors of global and local factors all converge into fixed solutions. The result establishes the global convergence of factors, as it does not require careful initialization. Second, the convergence rates for global and local factors differ. Global factors converge at a rate of $O(\frac{1}{T})$, which is standard in non-convex optimization. However, since some local factors must be perpendicular to the global factors, they converge at a slightly slower rate of $O(\frac{1}{\sqrt{T}})$. Third, our result is based on having $|\setK|\ge 2$. This requirement allows orthogonality to be maintained by each mode being updated.

To validate the convergence rates, we provide a proof-of-concept simulation study.  
Fig. \ref{fig:Loss_client} displays an example of the convergence of global and local factor matrices in this simulation. Both subspace errors of the local and global components go to 0.  
This slower rate of local components is primarily a result of the orthogonality requirement, and this result verifies Theorem \ref{thm:convergence}. The detail of this simulation study is relegated to Appendix \ref{app:ConvergenceSimulation}


\begin{figure}[H]
    \centering
    \includegraphics[width=0.7\linewidth]{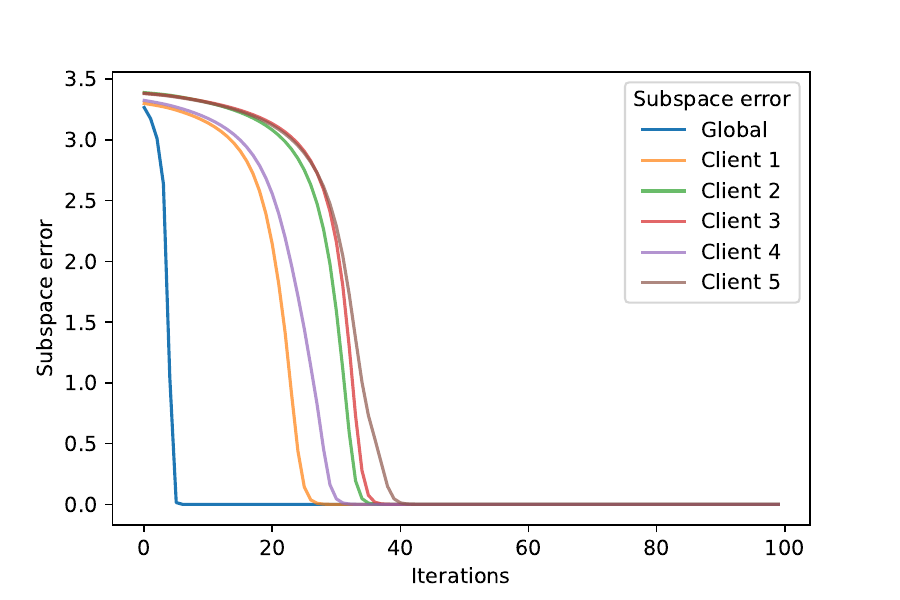}
    \caption{Subspace error for global component and different local sources}
    \label{fig:Loss_client}
\end{figure}


\subsection{Model initialization}\label{subsec:practical}



One simple approach is to initialize all components randomly. Alternatively, one may use a Tucker decomposition on all the data for initialization. To do so, let $s=\sum_{n=1}^N s_n$ be the total number of samples from all sources and recall that the data $\tsY_n$ from source $n$ has dimension $I_1\times\ldots\times I_K\times s_n$. Now the following steps can be taken: 
\begin{enumerate}
    \item Construct a tensor $\tsY$ that encompasses all samples from every source, with dimensions $I_1 \times \ldots \times I_K \times s$.
    \item Use Tucker decomposition $\tsY\approx\tsC\times_1\matU_{G,1}\ldots\times_K \matU_{G,K}$ to initialize global factors.
    \item Use Proposition \ref{prop:Core} to initialize the global core tensor for each source.
    \item For each source $n$, perform the Tucker decomposition on the local residual tensor $\tsR_{L,n}=\tsY_n-\tsC^{\star}_{G,n}\times_1\matU_{G,1}\ldots\times_K \matU_{G,K}$ to initialize the local core tensor and the local factor matrices.
\end{enumerate}
 This initialization does not require orthogonality between global and local components. Therefore, we may observe an increase in the reconstruction error in the first iteration. But we have found that in practice, this method yields faster convergence.

In addition, in our model, one needs to choose the modes in which orthogonality is imposed. At the same time, our theorem suggests that $|\setK|\ge 2$, and we found that the result is also true for $|\setK|= 1$ in various simulation studies. In practice, one can simply choose the mode with the largest dimension to impose the orthogonality constraint. Alternatively, cross-validation can be utilized to select the mode.

\subsection{Practical usage of \name \label{subsec:Usage}}
In this section, we introduce some practical applications of \name. Specifically, we shed light on its potential utility for improved classification, anomaly detection, and clustering. The key idea for all applications is to operate only on local components. This may allow for improved clustering, classification, and detection as differences become more explicit when shared knowledge is removed.

\subsubsection{Classification via \name}\label{subsubsec:Application_Classification}



To use \name for classification, we assume that each source corresponds to a class. Then we perform \name and can get the estimated local factor matrices $\hat{\matV}_{n,k},\, k=1,\ldots,K$ for each class. When a new piece of data $\tsY^{\text{new}}$ is sent, we can use the following decision rule to classify the new data.

    \begin{equation}
    \label{eqn:classifystatistics}
        \hat n = \arg\max_n \|\tsC^{\star}_{L,n}\|_F^2= \arg\max_n \|\tsY^{\text{new}}\times_1 \hat{\bm{V}}_{n,1}^\top\ldots \times_K\hat{\bm{V}}_{n,K}^\top\|_F^2.
    \end{equation}

The decision rule \eqref{eqn:classifystatistics} demonstrates that we can efficiently classify the data by selecting the class that maximizes the Frobenius norm of the local core tensor. This is because the largest norm of the core components indicates that the local subspace is most suitable for representing the original data, since the local core tensor is a projection of the original tensor onto the corresponding local subspace. 
We have found that such a decision rule is equivalent to finding the smallest possible reconstruction error across all classes. The discussion is relegated to Appendix \ref{app:ClassificationDiscussion}.

We want to emphasize that such a classification approach differs from traditional tensor-based classifiers \citep{klus2019tensor}, which directly trained supervised learning models for tensor classification. Here, we focus on a generative approach, which first trains $C$ data generation models (i.e., local subspaces) and then utilizes the representation error to decide to which class the data belong. The algorithm will construct global and local subspaces, which is beneficial not only for classification purposes but also for feature interpretation and visualization.



\subsubsection{Anomaly detection via \name}

By monitoring only local components, \name can improve anomaly detection methods as the changes in the underlying data become more explicit when common factors are removed. Specifically, we propose using $\| \tsC_L \|_F^2$ as the key monitoring statistic for online anomaly detection.


Here we emphasize that \name does not implement a sparsity penalty as often used in tensor-based anomaly detection \citep{yan2018real}. This is a unique benefit of \name as we do not assume that the anomaly patterns are sparse, which is too restrictive in some applications. As a result, \name can accommodate a wide range of anomalous pattern distributions. 


 \subsubsection{Clustering via \name}
 \name provides an alternative approach for client clustering based on local factors. Specifically, we focus on the setting of subspace clustering, which aims to cluster the clients if they are within the same local subspaces.
The subspace distance between client $n_1$ and $n_2$, $\rho_{n_1,n_2} = \|\hat{\bm{V}}_{n_1}\hat{\bm{V}}_{n_1}^\top - \hat{\bm{V}}_{n_2}\hat{\bm{V}}_{n_2}^\top\|_F^2$, can be calculated, where $\bm{V}_n$ is defined by the Kronecker product of the local factor matrices for client $n$. Then we can use spectral clustering to make clusters of the clients and further use multidimensional scaling to make the clustering plot \citep{hastie2009elements}.




\section{Numerical Studies}\label{sec:Simulation}
Now that we have introduced \name and its potential application, we validate its claimed advantages through numerical simulations. Sec. \ref{subsec:DataGeneration} introduces the data generation procedure. Sec. \ref{subsec:Performance}, Sec. \ref{subsec:Classification}, and Sec. \ref{subsec:Clustering} evaluates the performance of \name in terms of data reconstruction, classification, and clustering.

\subsection{Data generation}\label{subsec:DataGeneration}

In this simulation work, each sample of the data is a grayscale image with dimensions 50 by 50. The construction of each sample is low-rank global component, heterogeneous local component, and i.i.d standard normal noise, as in Eq. \eqref{eq:DataStructure}
\begin{equation}\label{eq:DataStructure}
\bm{\mathcal{Y}}_n = \bm{\mathcal{Y}}_{G,n} + \bm{\mathcal{Y}}_{L,n} + 
\bm{\mathcal E}_n.
\end{equation} 

We generate $N=3$ clients defined as $3$ patterns for the heterogeneous local component: Swiss pattern, oval pattern, and rectangle pattern, as in Fig. \ref{fig:PatternExample}. The value of all the patterns is 5 while the rest part is 0. In each pattern, we generate $10$ sample images. There is some variability within each pattern, as shown in the left part of Fig. \ref{fig:PatternExample}. The Swiss can be thin or thick; the oval can be vertical, horizontal, or circular; the rectangle can be wide, tall, or square.


For the global component, we randomly create orthonormal matrices $\matU_{G,1}$ and $\matU_{G,2}$ with dimension $50\times5$ for the 3 clients. Then we randomly generate the global core tensor $\bm{\mathcal C}_{G,n}$ with dimension $5\times5\times10$ for each client. And each entry of the global core tensors follows i.i.d. $N(0,100)$. 
The global components are constructed by $\bm{\mathcal{Y}}_{G,n} =\bm{\mathcal C}_{G,n}\times_1 \matU_{G,1}\times_2 \matU_{G,2}$, $n=1,2,3$.




Therefore, the full data dimension is $3\times50\times50\times10$.
Some examples of the data generation structure are shown in the right part of Fig. \ref{fig:PatternExample}. The three rows are for three patterns. The two columns show the global and local components, respectively, and the third column shows the sum of the global and local components along with the error term. With noise and global background, the local pattern can barely be recognized, which makes accurate identification of the local patterns challenging.


\begin{figure}
    \centering
    \begin{minipage}{0.45\linewidth}
        \centering
        \includegraphics[width=\linewidth]{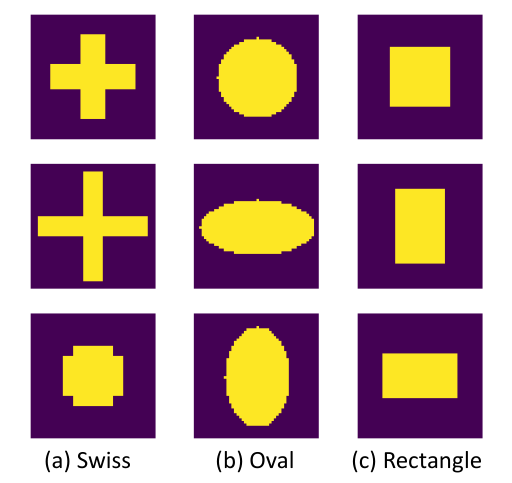}
    \end{minipage}
    \begin{minipage}{0.45\linewidth}
        \centering
        \includegraphics[width=\linewidth]{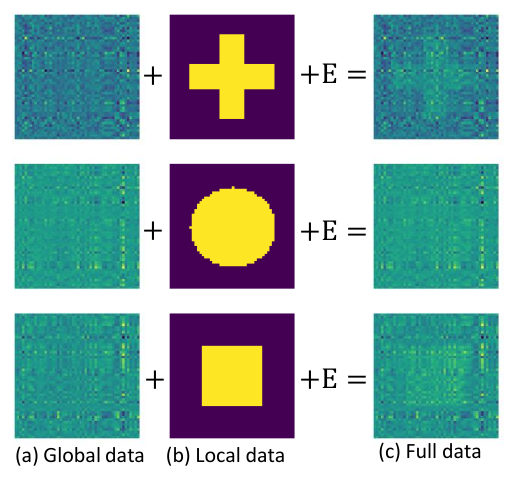}
    \end{minipage}
    \caption{\footnotesize\textbf{Left}: examples of variability in each pattern. 
    \textbf{Right}: examples of different components in each pattern.  }
    \label{fig:PatternExample}
\end{figure}

\subsection{Performance}\label{subsec:Performance}

In the generated data, we apply \name to decouple the global and local components. For comparison, we also evaluate the performance of some benchmark algorithms.


\begin{enumerate}
    \item \texttt{globalTucker}: We first concatenate the samples of all clients into tensor $\tsY$ and then apply the Tucker decomposition on $\tsY$. 
    \item \texttt{localTucker}: We apply a standard Tucker decomposition on each client $\tsY_n$ individually. 
    \item \texttt{robustTucker}: We first concatenate samples from all clients in $\tsY$, then apply the method in \citet{lu2019tensor} to identify the low-rank and sparse components.
    \item \texttt{perPCA}: We apply \texttt{perPCA} \citep{shi2022personalized} on the vectorized dataset where we vectorize $50 \times 50$ images into vectors of length $2500$ and use \texttt{perPCA} to find global and local components. Although \texttt{perPCA} is designed for vector datasets, this comparison can highlight the need for personalized Tensor decompositions when data is in tensor form.
\end{enumerate}




\begin{figure}[t]
    \centering
    \includegraphics[width=0.9\linewidth]{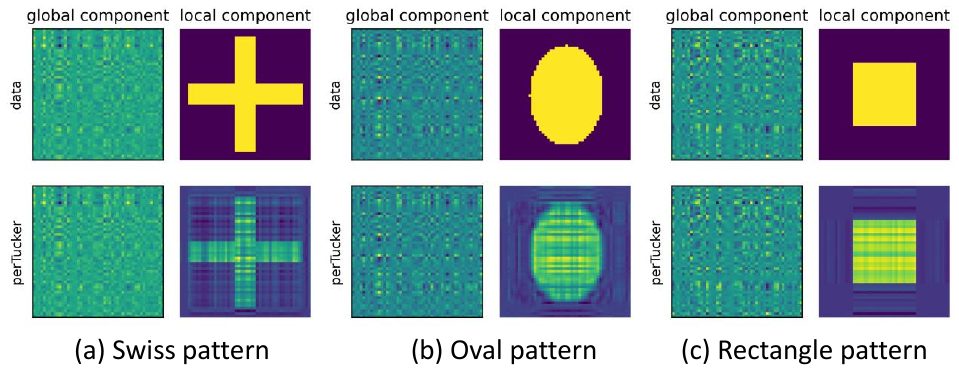}
    \caption{Reconstruction result example for three patterns}
    \label{fig:SimulationResult}
\end{figure}
 Fig. \ref{fig:SimulationResult} depicts examples of global and local component reconstruction via \name. The first row represents the data, and the second row shows the reconstruction from \name. The columns represent global and local components. The three examples with different patterns indicate that \name can effectively reconstruct the shared and unique data patterns.
 
 Furthermore, to numerically examine the performance and benchmark algorithms, we calculate a few performance metrics. $\hat{\matU}_{G,k}$ and $\hat{\matV}_{n,k}$ represent the estimated global and local factor matrices; $\bm{\mathcal{\hat{Y}}}_{G,n}$ and $\bm{\mathcal{\hat{Y}}}_{L,n}$ represent the estimated reconstruction of global and local components.
\begin{enumerate}
    \item \textbf{Global subspace error}: We compute the regularized global subspace error between the ground truth global factors $\{\matU_{G,k}\}$ and the estimated global factors $\{\hat{\matU}_{G,k}\}$ by  $\varrho (\bigotimes_{k=1}^K\matU_{G,k}, \bigotimes_{k=1}^K\bm{\hat{U}}_{G,k})/\|\bigotimes_{k=1}^K\matU_{G,k}\|_F^2 $.
    \item \textbf{Local subspace error}: We first generate 100 images for each pattern, and use Tucker decomposition to estimate the local factors $\{\matV_{n,k}\}$ for each pattern. Then the regularized local subspace error is calculated by $\varrho (\bigotimes_{k=1}^K\matV_{n,k}, \bigotimes_{k=1}^K\hat{\matV}_{n,k})/\|\bigotimes_{k=1}^K\matV_{n,k}\|_F^2 $, and take the average of $3$ patterns.
    \item \textbf{Global component error}: defined by $\small\sum_n\|\bm{\mathcal{\hat{Y}}}_{G,n} - \bm{\mathcal{Y}}_{G,n} \|_F^2 / \sum_n \|\bm{\mathcal{Y}}_{G,n}\|_F^2$. 
    \item \textbf{Local component error}: defined by $\small\sum_n\|\bm{\mathcal{\hat{Y}}}_{L,n} - \bm{\mathcal{Y}}_{L,n} \|_F^2 / \sum_n \|\bm{\mathcal{Y}}_{L,n}\|_F^2$.
    \item \textbf{Denoised error}: defined by  $\small\sum_n \| (\bm{\mathcal{Y}}_{G,n} + \bm{\mathcal{Y}}_{L,n}) - (\bm{\mathcal{\hat{Y}}}_{G,n} + \bm{\mathcal{\hat{Y}}}_{L,n})\|_F^2) / \sum_n \|\bm{\mathcal{Y}}_{G,n} + \bm{\mathcal{Y}}_{L,n}\|_F^2 $ 
\end{enumerate}
We run each experiment 10 times from $10$ different random seeds and report their mean and standard deviation. The results of the measuring statistics for different methods are summarized in Table \ref{tab:Error_table}. From Table \ref{tab:Error_table}, we can conclude the following statements. 

\begin{table}[t]
    \caption{Component-wise reconstruction error with standard deviation in parenthesis}
    \centering
    \scriptsize
    \begin{tabular}
    {cccccc}
    \hline
         & \name & \texttt{perPCA} & \texttt{globalTucker} & \texttt{localTucker} & \texttt{robustTucker} \\\hline
       Global subspace error ($10^{-3}$) & $\textbf{2.3}( 0.8)$ & $536(6)$ & $\textbf{2.4} (0.8)$ & N/A & $3.7 (0.9)$\\
       Local subspace error $(10^{-1})$& $\textbf{6.4}(0.4)$ & $9.88(0.07)$ & N/A & $9.34(0.03)$ & N/A \\
       Global component error $(10^{-3})$ & $\textbf{5.7}( 1.6)$ & $372(16)$& $\textbf{5.7}(1.6)$ & N/A & $264 (12)$\\
       Local component error $(10^{-1})$ & $\textbf{2.8}( 0.6)$ & $23.1( 1.3)$ & N/A & $63(3)$ & $10(0.09)$\\
       Denoised error ($10^{-2}$) & $4(0.7)$ & $13.7(0.7)$ & $14.1(0.7)$ & $\textbf{2} (2)$ & $13.7(0.7)$\\
        \hline
    \end{tabular}
    \label{tab:Error_table}
\end{table}

1) \name yields the best results in separating the global and local components due to the utilization of the low-rank tensor structure of both components, from the following observations: (a) Comparing the global component error of \name and \texttt{perPCA}, we can conclude that \name identifies the global component with better accuracy due to its use of low-rank tensor structures. (b) Comparing the local component error of \name and \texttt{robustTucker}, although \texttt{robustTucker} identifies the global subspace with decent accuracy, it yields a much larger error in terms of local component reconstruction due to the fact that the local component does not assume any low-rank structure.

2) The reconstruction error for local components is larger for all methods compared to the global components. Two factors contribute to this: (a) the reconstruction of the local component is generally harder compared to the global components since it is only shared within the same clients, resulting in lower accuracy and larger variance with the use of fewer datasets; (b) the local components are generated by the shape variations of Swiss, oval, and rectangle patterns, which are not exactly low-rank, thus, the true local rank subspace is also an approximation from the dataset.




\subsection{Classification}\label{subsec:Classification}
In this section, we evaluate the classification performance of the \name algorithm.
The training sample size ranges from 10 to 50 with a step of 10 for each pattern. Then we perform the \name decomposition and get the local factor matrices. Next, we create 50 new images for each pattern and use the method described in Sec. \ref{subsubsec:Application_Classification} to classify the 150 images. This procedure is repeated 100 times. Table \ref{tab:Classification} displays the mean and standard deviation of the classification accuracy for various parameters.	
From Table \ref{tab:Classification} we can see that \name exhibits excellent classification performance even when the sample size is small. Increasing the training sample size leads to improved prediction accuracy. In comparison, if we perform local Tucker on the three clients and classify the new figures, the accuracy will be around 33\% regardless of the training sample size. This accuracy is close to ``guessing" the class. This is because local Tucker will model the commonality and peculiarity simultaneously. Therefore, global features are also considered when making the decision, while the multiplication coefficients are randomly generated.
\begin{table}[t]
    \centering
    \caption{Classification accuracy and the standard deviation for different parameters}
    \begin{tabular}{cccccc}
    \hline
        Training sample size & 10 & 20 & 30 & 40 & 50\\\hline
    Accuracy & 86.7\% & 91.2\% & 92.7 \% & 93.9\% & 95.0\%\\
    SD of accuracy & 6 \% & 4\% & 4\% & 4\% & 3\%\\
    \hline
       \end{tabular}
    
    \label{tab:Classification}
\end{table}

To visualize the classification process, we use the boxplot to show the summary of the test statistics in \eqref{eqn:classifystatistics} in Fig. \ref{fig:class_box}. The sample size is set to 50. The test statistics are centralized by the median statistics within each pattern. The test statistics for each specific true pattern are significantly higher than those for the other two patterns in the corresponding cases, implying that the classification procedure is effective and robust.


\begin{figure}[t]
    \centering\includegraphics[width=0.7\linewidth]{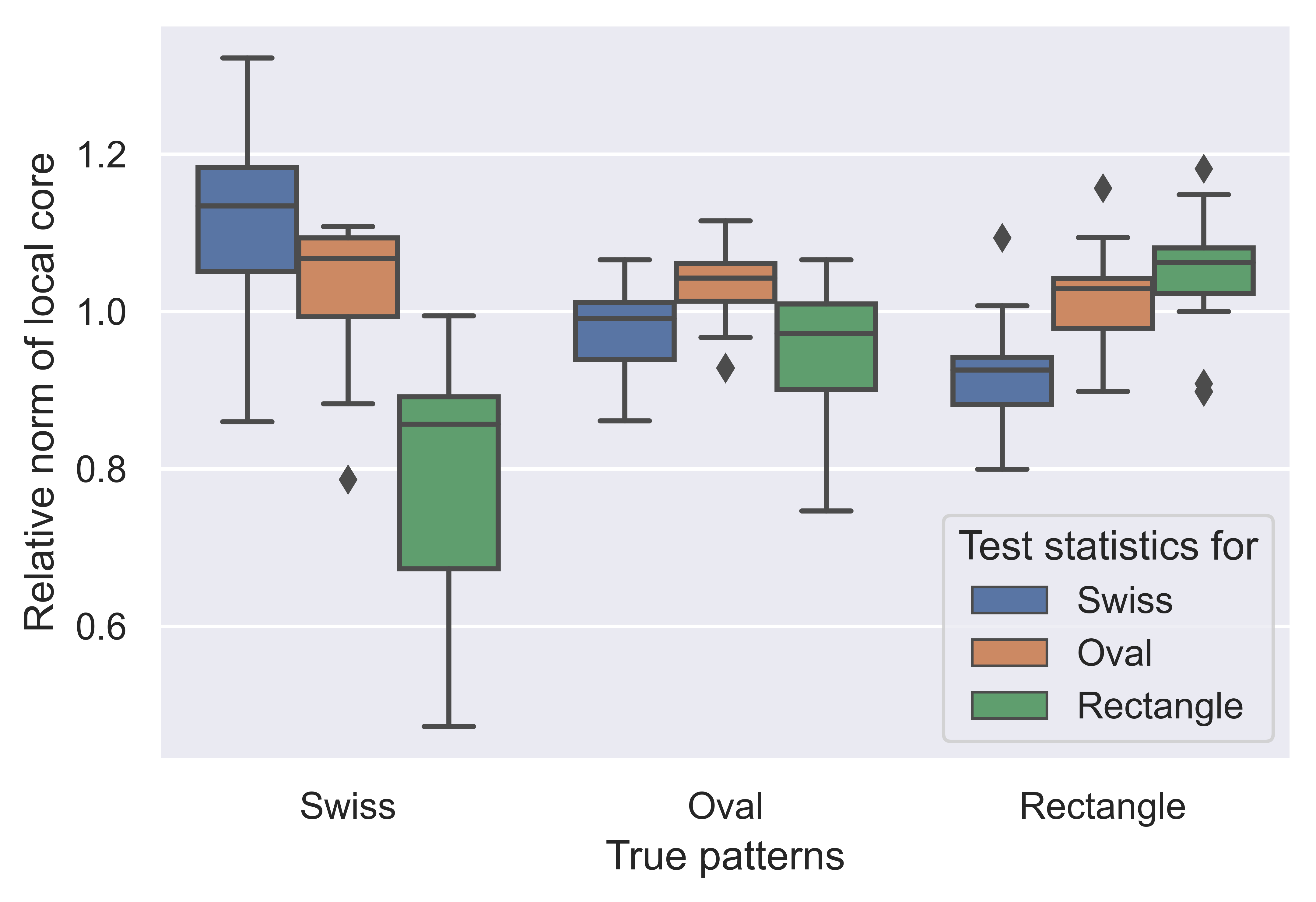}
    \caption{Box plot for test statistics defined in \eqref{eqn:classifystatistics} for different patterns and different classes}
    \label{fig:class_box}
\end{figure}

\subsection{Clustering}\label{subsec:Clustering}
 In this section, we study the clustering performance of \name under a different problem setting. When generating the data, the variability within each pattern is measured by a ``ratio" variable. The performance of clustering is evaluated by the ability to cluster similar ``ratio" within each pattern. We focus on the Swiss pattern and cut the range of ``ratio" from $0.7$ to $1.4$ into the group of $7$ ratio intervals, each with a width of $0.1$. Then for each ratio interval, we generate $3$ clients, each with a sample size of $100$ figures. We make clustering plots for the $21$ clients and see the aggregation of the clients with similar ratio, as shown in Fig. \ref{fig:clustering_result}. 



\begin{figure}[H]
    \centering
        \includegraphics[width=0.7\linewidth]{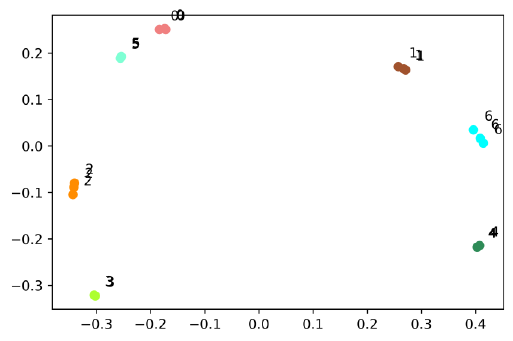}
    \caption{Clustering plot for swiss patterns}
    \label{fig:clustering_result}
\end{figure}


\section{Case Study} \label{sec:CaseStudy}
In this section, we use two case study experiments to demonstrate the application of the \name algorithm. Sec. \ref{subsec:Solar} provides an anomaly detection example that illustrates the power of the \name algorithm in detecting solar flares. 
The subsequent section, Sec. \ref{subsec:tonnage}, demonstrates the effectiveness of the \name technique in classifying tonnage fault signals.	

\subsection{Anomaly detection in solar flare data}\label{subsec:Solar}
The first example involves monitoring solar activities and detecting solar flares from a stream of solar images. A solar flare is a significant event in the sun that emits enormous amounts of energetic charged particles that could cause power grids to fail \citep{marusek2007solar}. Therefore, detecting solar flares promptly is critical to implementing preemptive and corrective measures. However, monitoring solar activities is challenging due to the high dimensionality of the solar thermal imaging data and the gradual changes in solar temperature over time. Existing detection methods, which rely on subtracting the functional mean (background) using the sample mean, are insufficient to detect small transient flares in the dynamic system \citep{zhao2022adaptive}. Other studies focus on the decomposition of solar images into their background and anomaly components \citep{yan2018real}.

This dataset, publicly available in \citep{zhao2022adaptive}, comprises a sequence of images of size 232 × 292 pixels captured by a satellite. We use a sample of 300 frames in this case study. To detect solar flares in real time, we begin by subtracting the mean from each sample and then preprocess the data using the method proposed by \citep{aharon2006k}. Following this, we use a sliding window of $8\times8$ to divide each frame into small patches,  resulting in a total of 1044 patches. The four right-most columns of pixels are discarded. Each patch or tile is then vectorized, yielding a data dimension of $300\times1044\times64$.
After the preprocessing step, we apply the \name algorithm to the data to break them down into two components: global components representing the slowly changing background and local components indicating the detected solar flares.

\begin{figure}[t]
    \centering
    \includegraphics[width=0.9\linewidth]{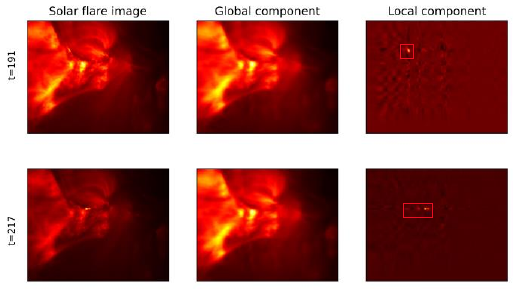}
    \caption{Detection of solar flare at $t=191$ and $t=217$}
    \label{fig:solar_detect}
\end{figure}
Fig. \ref{fig:solar_detect} shows two frames where there is an abrupt change in the almost-stationary background. On the original images, such changes are not visible in the complicated background. However, after using \name to extract global and local components, one can clearly see the location and movement of small and rapid changes in local components. The experiments highlight \name's ability to change the signal magnification.

Moreover, since the local component represents the anomaly, we can use the Frobenius norm of the local core tensor $\norm{\tsC_{L,n}}_F$ as monitoring statistics to detect the anomaly. The results are shown in Fig. \ref{fig:solar_core}, where we plot the logarithm of the Frobenius norm of the local core tensor $\norm{\tsC_{L,n}}_F$ for $300$ frames.
\begin{figure}[t]
    \centering
    \includegraphics[width=0.7\linewidth]{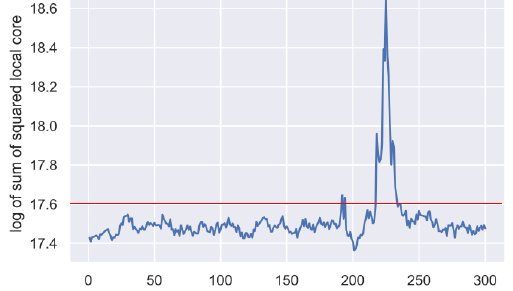}
    \caption{The logarithm of Frobenius norm of local core tensors for $300$ frames}
    \label{fig:solar_core}
\end{figure}
For the most time, the curve in Fig. \ref{fig:solar_core} remains at a low level, suggesting that no significant changes in solar activities are detected. However, there are 2 clear peaks above the red control line, suggesting that there are two solar flares. Therefore, Fig. \ref{fig:solar_core} provides an intuitive approach for detecting anomalous behaviors.

\subsection{Classification for tonnage data} \label{subsec:tonnage}
\name also applies to monitoring tonnage profiles in a multi-operation forging process that utilizes multiple strain gauge sensors. Specifically, the forging process employs four columns, each equipped with a strain gauge sensor that measures the tonnage force exerted by the press uprights, as illustrated in Fig. \ref{fig:tonnagesetup}. Consequently, each operational cycle generates a four-channel tonnage profile. The dataset for this case study consists of 305 in-control profiles, collected under normal production conditions, and 69 out-of-control profiles for each of the four different fault classes. The length of each channel profile is 1201, resulting in a data dimension of $\tsY^{4 \times 1201  \times305}$ for the ``normal'' class and $\tsY^{4 \times 1201 \times69}$ for each fault class. Fig. \ref{fig:tonnage} presents examples of profiles for both normal and the four different fault conditions.

\begin{minipage}{0.48\linewidth}
    \begin{center}
        \includegraphics[width=\linewidth]{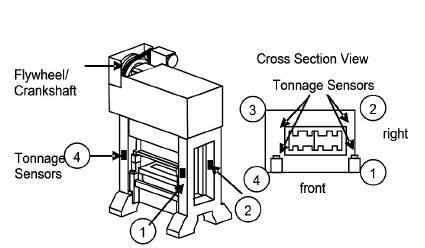}
        \captionof{figure}{Tonnage Signal Monitoring}
        \label{fig:tonnagesetup} 
    \end{center}
\end{minipage}
\begin{minipage}{0.48\linewidth}
    \begin{center}
    \includegraphics[width=0.9\linewidth]{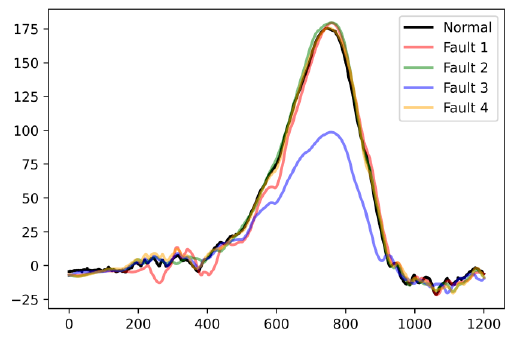}
    \captionof{figure}{Tonnage Data}
    \label{fig:tonnage}
    \end{center}
\end{minipage}


\begin{figure}[h]
    \centering
    \includegraphics[width=0.9\linewidth]{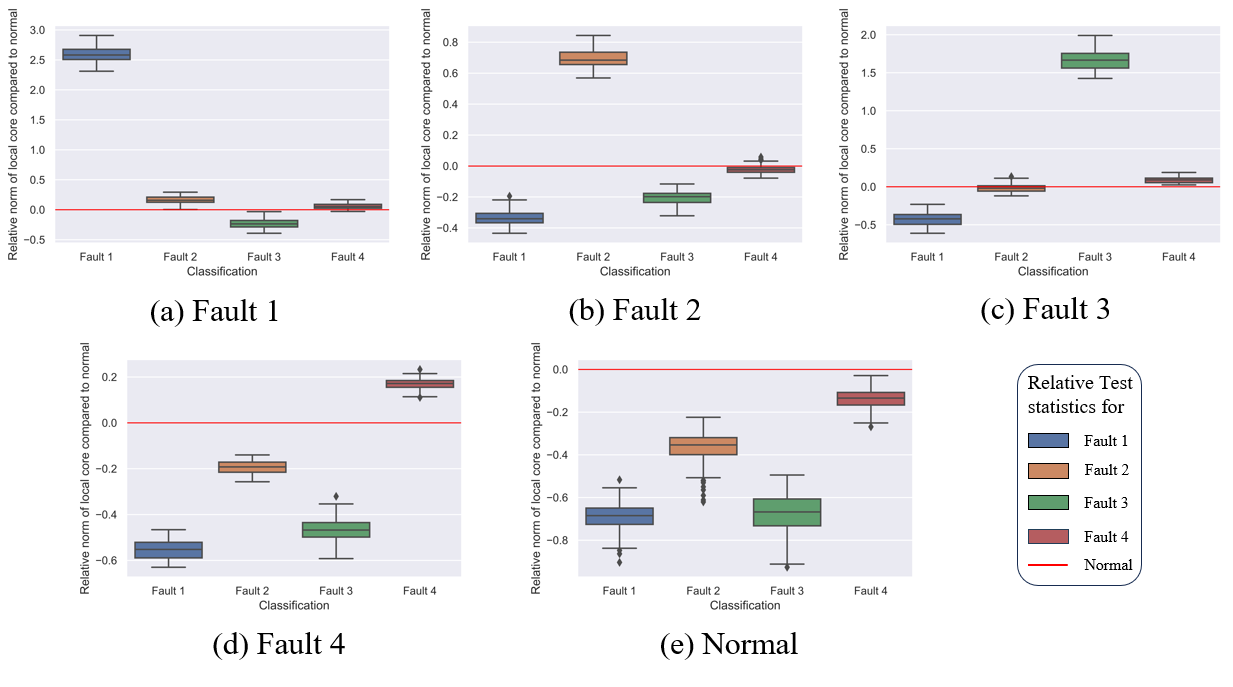}
    \caption{Relative decision test statistics compared to normal when the true label is (a) Fault 1, (b) Fault 2, (c) Fault 3, (d) Fault 4, (e) Normal}
    \label{fig:tonnage_box}
\end{figure}



In this case study, we define the five ``clients" as normal and four fault conditions. We select the first 50 normal samples and the first 10 samples from each fault condition to form the training dataset. The test dataset comprises $255$ normal samples and $59$ samples for each fault condition. We further assume that only the global and local factor matrices along the signal length dimension are orthogonal to each other.

Applying the classification method in Sec. \ref{subsubsec:Application_Classification}, we are able to achieve 100\% accuracy for classifying all fault modes and normal cases. Fig. \ref{fig:tonnage_box} shows the relative test statistics by computing the difference between the test statistics of the corresponding fault modes and the normal cases. Consequently, normal test statistics have been normalized as $y=0$ for each sample, as shown in a red line in Fig. \ref{fig:tonnage_box}. 

We can see that in every fault mode, the range of the test statistics for the corresponding fault mode is much higher than those of the other three faults and the normal baseline. When the true label is normal, the relative test statistics for all fault modes are less than the normal baseline. Thus, \name provides informative statistics for anomaly detection and fault diagnostics. 


\section{Conclusion} \label{sec:Conclusion}
In this paper, we propose the \name method to decompose the tensor data into global components with shared factor matrices and local components that are orthogonal to global components. Global and local components model the commonality and peculiarity of the data. Subsequently, we present an efficient algorithm to solve the model, based on the block coordinate descent algorithm. The inclusion of proximal terms in the step of updating factor matrices allows the convergence to the stationary point. Then, several applications of \name, such as anomaly detection, classification, and clustering, are demonstrated by simulation and case studies. 

Some future studies will include some special properties of the global and local components. For example, the local components can be sparse to model the anomaly components in different applications.

\newpage

\bibliographystyle{apalike} 
\bibliography{reference} 

\appendix      

\section{Proof of Proposition \ref{prop:orthogonal}}
In this section, we prove the proposition of the equivalent condition for two tensors to be orthogonal to each other. 
We first introduce Lemma \ref{lemma:kron-norm}.

\begin{lemma}\label{lemma:kron-norm} For two matrices $\matA$ and $\matB$,
    $$
    \|\bm{A}\bigotimes\bm{B}\|_F^2=\|\bm{A}\|_F^2\|\bm{B}\|_F^2
    $$
\end{lemma}
\begin{proof}
    \begin{align*}
        \|\bm{A}\bigotimes\bm{B}\|_F^2&= Tr\left[ (\bm{A}\bigotimes\bm{B})^\top(\bm{A}\bigotimes\bm{B})\right]\\
        &=Tr[\bm{A}^\top\bm{A}\bigotimes\bm{B}^\top\bm{B}]\\
        &=Tr(\bm{A}^\top\bm{A})Tr(\bm{B}^\top\bm{B})\\
        &=\|\bm{A}\|_F^2\|\bm{B}\|_F^2
    \end{align*}
\end{proof}
Then we prove Proposition \ref{prop:orthogonal}
\begin{proof}

\begin{align*}\left\langle \bm{\mathcal{Y}}_{G,n},\bm{\mathcal{Y}}_{L,n}\right\rangle  & =\left\langle \bm{\mathcal{C}}_{G,n}\times_{1}\matU_{G,1}\ldots\times_{K}\matU_{G,K},\bm{\mathcal{C}}_{L,n}\times_{1}\bm{V}_{n,1}\ldots\times_{K}\bm{V}_{n,K}\right\rangle \\
 & =\left\langle \left(\matU_{G,K}\bigotimes\cdots\bigotimes\matU_{G,1}\right)\bm{c}_{G,n},\left(\bm{V}_{n,K}\bigotimes\cdots\bigotimes\bm{V}_{n,1}\right)\bm{c}_{L,n}\right\rangle \\
 & =Tr(\bm{c}_{G,n}^{T}\left(\matU_{G,K}^{T}\bigotimes\cdots\bigotimes\matU_{G,1}^{T}\right)\left(\bm{V}_{n,K}\bigotimes\cdots\bigotimes\bm{V}_{n,1}\right)\bm{c}_{L,n})\\
 & =Tr(\bm{c}_{G,n}^{T}\left(\matU_{G,K}^{T}\bm{V}_{n,K}\bigotimes\cdots\bigotimes\matU_{G,1}^{T}\bm{V}_{n,1}\right)\bm{c}_{L,n}),
\end{align*}
where $\bm{c}_{G,n}$ and $\bm{c}_{L,n}$ are vectorized global and local tensors.
    
    On the one hand, if $\matU_{G,k}^\top\bm{V}_{n,k}=0$ for some $k$, $\matU_{G,K}^\top\bm{V}_{n,K}\bigotimes\cdots\bigotimes \matU_{G,1}^\top\bm{V}_{n,1}=\bm{0}$ and thus $\left \langle \bm{\mathcal Y}_{G,n}, \bm{\mathcal Y}_{L,n} \right \rangle=0$ for any $\bm{\mathcal C}_{G,n}$, $\bm{\mathcal C}_{L,n}$.
    On the other hand, if we want $\left \langle \bm{\mathcal Y}_{G,n}, \bm{\mathcal Y}_{L,n} \right \rangle=0$ for any $\bm{\mathcal C}_{G,n}$, $\bm{\mathcal C}_{L,n}$, we need 
    \begin{align*}
    &\matU_{G,K}^\top\bm{V}_{n,K}\bigotimes\cdots\bigotimes \matU_{G,1}^\top\bm{V}_{n,1}=\bm{0}\\
    \Leftrightarrow & \|\matU_{G,K}^\top\bm{V}_{n,K}\bigotimes\cdots\bigotimes \matU_{G,1}^\top\bm{V}_{n,1}\|_F^2=0\\
    \Leftrightarrow & \|\matU_{G,K}^\top\bm{V}_{n,K}\|_F^2 \cdots \|\matU_{G,1}^\top\bm{V}_{n,1}\|_F^2=0\\
    \end{align*}
    Therefore, there exists some mode $k$, $\|\matU_{G,k}^\top\bm{V}_{n,k}\|_F^2$ is $0$, which means that   $\matU_{G,k}^\top\bm{V}_{n,k}=0$.
\end{proof}

\section{Proof of Proposition \ref{prop:Core}}

\begin{proof}
    For each $n$, the optimization problem to derive global and local core tensors is
    \begin{equation}
        \min_{\bm{\mathcal C}_{G,n}, \bm{\mathcal C}_{L,n}} \|\bm{\mathcal Y}_n - \bm{\mathcal C}_{G,n}\times_1 \matU_{G,1}\ldots\times_K \matU_{G,K}- \bm{\mathcal C}_{L,n}\times_1 \bm{V}_{n,1} \ldots\times_K \bm{V}_{n,K}\|_F^2
    \end{equation}
    By vectorizing the terms, we get 
    \begin{align*}
         &\min_{\bm{c}_{G,n}, \bm{c}_{L,n}} \|\bm{y}_n-(\bigotimes_{k=K}^1 \matU_{G,k})\bm{c}_{G,n}-(\bigotimes_{k=K}^1 \bm{V}_{n,k})\bm{c}_{L,n}\|^2 \\
         \iff & \min_{\bm{c}_{G,n}, \bm{c}_{L,n}} \norm{\bm{y}_n-
         \begin{pmatrix}
         \bigotimes_{k=K}^1 \matU_{G,k} & \bigotimes_{k=K}^1 \bm{V}_{n,k}
         \end{pmatrix}
         \begin{pmatrix}
             \bm{c}_{G,n}\\
             \bm{c}_{L,n}
         \end{pmatrix}
         }^2
    \end{align*}
    where the sign $\bigotimes_{k=K}^1$ represents the Kronecker product of the matrices in the order $K,K-1,\ldots,1$. $\bm{c}_{G,n}$ and $\bm{c}_{L,n}$ represents the vectorized global and core tensor from source $n$.
    This is a least square loss in linear regression. The solution for $\begin{pmatrix}
             \bm{c}_{G,n}\\
             \bm{c}_{L,n}
         \end{pmatrix}$ is 
         $$\scriptsize
         \begin{pmatrix}
             \bm{c}^{\star}_{G,n}\\
             \bm{c}^{\star}_{L,n}
         \end{pmatrix} = \left[ \begin{pmatrix}
         \bigotimes_{k=K}^1 \matU_{G,k} & \bigotimes_{k=K}^1 \bm{V}_{n,k}
         \end{pmatrix}^\top \begin{pmatrix}
         \bigotimes_{k=K}^1 \matU_{G,k} & \bigotimes_{k=K}^1 \bm{V}_{n,k}
         \end{pmatrix}\right]^{-1} \begin{pmatrix}
         \bigotimes_{k=K}^1 \matU_{G,k} & \bigotimes_{k=K}^1 \bm{V}_{n,k}
         \end{pmatrix}^\top \bm{y}_n
         $$
         Since the global and local factor matrices are orthogonal, $(\bigotimes_{k=K}^1 \matU_{G,k})^\top \bigotimes_{k=K}^1 \matU_{G,k}=I$, $( \bigotimes_{k=K}^1 \bm{V}_{n,k})^\top \bigotimes_{k=K}^1 \bm{V}_{n,k}=I$. The upper off-diagonal matrix in the product is 
         $$
         (\bigotimes_{k=K}^1 \bm{V}_{n,k})^\top \bigotimes_{k=K}^1 \matU_{G,k} = \bigotimes_{k=K}^1 \bm{V}_{n,k}^\top\matU_{G,k}=0
         $$
         where the last equality holds because in at least one dimension, $\bm{V}_{n,k}^\top\matU_{G,k}=0$, similarly, the lower off-diagonal matrix is also $0$, therefore, the product is the identity matrix.
         $$
         \begin{pmatrix}
             \bm{c}_{G,n}^{\star}\\
             \bm{c}_{L,n}^{\star}
         \end{pmatrix}=\begin{pmatrix}
             (\bigotimes_{k=K}^1 \matU_{G,k}^\top) \bm{y}_n\\
             (\bigotimes_{k=K}^1 \bm{V}_{n,k}^\top) \bm{y}_n
         \end{pmatrix}
         $$
         Transform the solution above back to the tensor, and we get the solution in the proposition.
\end{proof}

\section{Proof of Lemma \ref{lemma:subspace-equiv}}
\begin{proof}
    \begin{align*}
        \|\bm{UU}^\top-\matU_t\matU_t^\top\|_F^2 &= Tr \left[(\bm{UU}^\top-\matU_t\matU_t^\top) (\bm{UU}^\top-\matU_t\matU_t^\top)^\top \right]\\
        &= Tr\left[\bm{UU}^\top\bm{UU}^\top -2\bm{UU}^\top\matU_t\matU_t^\top +\matU_t\matU_t^\top\matU_t\matU_t^\top\right]\\
        &=2c-2Tr\left[\matU^\top\matU_t\matU_t^\top \matU\right]
    \end{align*}
    where $c$ is the number of rows of the factor matrix $\matU$.
\end{proof}

\section{Proof of Lemma \ref{lemma:min-to-max}}
\begin{proof}
    The proof of Equation \eqref{eq:global-equiv} and Equation \eqref{eq:local-equiv} are similar. Here we take one component in Equation \eqref{eq:global-equiv} as an example. We use $\tsC^{\star}_{G,n}$ and $\tsC^{\star}_{L,n}$ to represent the optimized core tensors.
    \begin{align*}
        & \|\tsR_{G,n} - \tsC^{\star}_{G,n} \times_1 \matU_{G,1}\ldots\times_K \matU_{G,K} \|_F^2 \\
        =& \|\tsR_{G,n}\|_F^2 -2\left\langle \tsR_{G,n},\tsC^{\star}_{G,n} \times_1 \matU_{G,1}\ldots\times_K \matU_{G,K} \right\rangle +\|\tsC^{\star}_{G,n} \times_1 \matU_{G,1}\ldots\times_K \matU_{G,K}\|_F^2\\
        =&\|\tsR_{G,n}\|_F^2 -2\left\langle\tsR_{G,n}\times_1 \matU_{G,1}^\top \ldots \times_K \matU_{G,K}^\top, \tsC^{\star}_{G,n} \right\rangle +\|\tsC^{\star}_{G,n}\|_F^2
    \end{align*}
        The second term can be formulated as
        \begin{small}
        \begin{align*}
            &-2\left\langle\tsR_{G,n}\times_1 \matU_{G,1}^\top\ldots\times_K \matU_{G,K}^\top, \tsC^{\star}_{G,n}\right\rangle\\
            =& -2\left\langle\tsR_{G,n}\times_1 \matU_{G,1}^\top\ldots \times_K \matU_{G,K}^\top, \tsR_{G,n}\times_1 \matU_{G,1}^\top \ldots \times_K \matU_{G,K}^\top\right\rangle\\
            =&-2\|\tsR_{G,n}\times_1 \matU_{G,1}^\top \ldots \times_K \matU_{G,K}^\top\|_F^2
        \end{align*}
        \end{small}
        And the third term can be formulated as
        \begin{align*}
            \|\tsC^{\star}_{G,n}\|_F^2 &=\|(\tsR_{G,n}+\tsC_{L,n}\times_1\matV_{n,1}\ldots\times_K\matV_{n,K}) \times_1 \matU_{G,1}^\top\ldots\times_K \matU_{G,K}^\top\|_F^2\\
            & = \|\tsR_{G,n}\times_1 \matU_{G,1}^\top\ldots \times_K \matU_{G,K}^\top\|_F^2
        \end{align*}
        Therefore, for each $n=1,\ldots,N$,
        \begin{align*}
            \|\tsR_{G,n} - \tsC^{\star}_{G,n} \times_1 \matU_{G,1}\ldots\times_K \matU_{G,K}\|_F^2 = \|\tsR_{G,n}\|_F^2 - \|\tsR_{G,n}\times_1 \matU_{G,1}^\top\ldots \times_K \matU_{G,K}^\top\|_F^2 
        \end{align*}
        and Equation \eqref{eq:global-equiv} holds. Similarly, for each $n=1,\ldots,N$,
        \begin{align*}
            \|\tsR_{L,n} - \tsC^{\star}_{L,n} \times_1 \bm{V}_{n,1}\ldots\times_K \bm{V}_{n,K}\|_F^2 = \|\tsR_{L,n}\|_F^2 - \|\tsR_{L,n}\times_1 \bm{V}_{n,1}^\top\ldots \times_K \bm{V}_{n,K}^\top\|_F^2 
        \end{align*}
        and Equation \eqref{eq:local-equiv} holds.
\end{proof}

\section{Proof of Proposition \ref{prop:Global_prox}}\label{App:proof_global_prox}
\begin{proof}
    First, by Lemma \ref{lemma:subspace-equiv}, we have
    \begin{align*}
        \|\matU_{G,k}\matU_{G,k}^\top-\matU_{G,k,t}{\matU_{G,k,t}^\top}\|_F^2 =2c-2Tr\left[\matU_{G,k}^\top\matU_{G,k,t}{\matU_{G,k,t}^\top} \matU_{G,k}\right]
    \end{align*}
    where $c$ is the number of rows of the factor matrix, therefore, can be omitted in the optimization.
    Then by Lemma \ref{lemma:min-to-max}, we can rewrite the minimization problem into a maximization problem.
    \begin{align*}
        &\min_{\matU_{G,k}} \sum_{n=1}^N \|\tsR_{G,n} - \tsC^{\star}_{G,n} \times_1 \matU_{G,1}\ldots\times_K \matU_{G,K} \|_F^2 +\rho \|\matU_{G,k}\matU_{G,k}^\top-\matU_{G,k,t}{\matU_{G,k,t}^\top}\|_F^2\\
        \iff & \max_{\matU_{G,k}} \sum_{n=1}^N \|\tsR_{G,n}\times_1 \matU_{G,1}^\top\ldots\times_K \matU_{G,K}^\top \|_F^2-\rho \|\matU_{G,k}\matU_{G,k}^\top-\matU_{G,k,t}{\matU_{G,k,t}^\top}\|_F^2\\
        \iff & \max_{\matU_{G,k}} \sum_{n=1}^N \|\matU_{G,k} (\tsR_{{G,n}})_{(k)}(\bigotimes_{q\neq k} \matU_{G,q}^\top)^\top\|_F^2+2\rho Tr\left[\matU_{G,k}^\top\matU_{G,k,t}{\matU_{G,k,t}^\top} \matU_{G,k}\right]\\
        \iff & \max_{\matU_{G,k}} \sum_{n=1}^N Tr\left[ \matU_{G,k}^\top \bm{W}_{G,n} \bm{W}_{G,n}^\top \matU_{G,k} \right] + 2\rho Tr\left[\matU_{G,k}^\top\matU_{G,k,t}{\matU_{G,k,t}^\top} \matU_{G,k}\right]\\
        \iff & \max_{\matU_{G,k}} Tr\left[ \matU_{G,k}^\top\left(\sum_{n=1}^N \bm{W}_{G,n} \bm{W}_{G,n}^\top + 2\rho \matU_{G,k,t}{\matU_{G,k,t}^\top} \right) \matU_{G,k}\right]
    \end{align*}

    Therefore, the columns of $\matU_{G,k}$ are the eigenvectors of $\sum_{n=1}^N \bm{W}_{G,n} \bm{W}_{G,n}^\top + 2\rho \matU_{G,k,t}{\matU_{G,k,t}^\top}$ corresponding to top $g_k$ eigenvalues.
\end{proof}

\section{Proof of Proposition \ref{prop:Local_prox}}\label{App:proof_local_prox}
We first introduce a Lemma
\begin{lemma}\label{lemma:orthgonal-opt}
    Suppose $\bm{V}\in\mathbb R ^{a\times b}$. $\matU\in\mathbb R ^{a\times c}$ is an orthonormal matrix, $\bm{S}\in\mathbb R^{a\times a}$ is a symmetric matrix. The solution to the optimization problem
    $$
    \max_{\bm V} Tr(\bm{V}^\top\bm{SV})~s.t.~\bm{V}^\top\bm{V}=\bm{I},~\bm{V}^\top\matU=0
    $$
    is the eigenvectors of the matrix $(\bm{I}-\bm{UU}^\top)\bm{S}(\bm{I}-\bm{UU}^\top)$ corresponding to the top $b$ eigenvalues.
\end{lemma}
\begin{proof}
    Since all the feasible solutions satisfy $\bm{V}^\top\matU=0$, we have
    \begin{align*}
        &\max_{\bm V} Tr(\bm{V}^\top\bm{SV})\\
        \iff & \max_{\bm V} Tr[\bm{V}^\top\bm{SV}-\bm{V}^\top\bm{UU}^\top\bm{SV}-\bm{V}^\top\bm{SUU}^\top\bm{V}+\bm{V}^\top\bm{UU}^\top\bm{S}\bm{UU}^\top\bm{V}]\\
        \iff &  \max_{\bm V} Tr[\bm{V}^\top(\bm{I}-\bm{UU}^\top)\bm{S}(\bm{I}-\bm{UU}^\top)\bm{V}]
    \end{align*}
    We claim that the optimal solution of 
    $$
    \max_{\bm V} Tr[\bm{V}^\top(\bm{I}-\bm{UU}^\top)\bm{S}(\bm{I}-\bm{UU}^\top)\bm{V}]~s.t.~\bm{V}^\top\bm{V}=\bm{I},~\bm{V}^\top\matU=0
    $$
    is the same by relaxing the condition $\bm{V}^\top\matU$. When this condition is relaxed, the optimal solution consists of the eigenvectors of the matrix $(\bm{I}-\bm{UU}^\top)\bm{S}(\bm{I}-\bm{UU}^\top)$. For any eigenvector $\bm{x}$, we have
    $$
    \lambda \bm{x}^\top\matU=\bm{x}^\top(\bm{I}-\bm{UU}^\top)\bm{S}(\bm{I}-\bm{UU}^\top)\matU=0
    $$
    Therefore, the optimal solution of the relaxed problem also satisfies $\bm{V}^\top\matU=0$. This completes the proof.
\end{proof}

Then we provide the proof of Proposition \ref{prop:Local_prox}.
\begin{proof}
    For each $n=1,\ldots,N$ and $k=1,\ldots,K$, the optimization problem to update the local factor matrix $\bm{V}_{n,k}$ is 
    \begin{equation}
        \bm{V}_{n,k}=\arg\min_{\bm{V}_{n,k}} \|\tsR_{L,n} - \tsC^{\star}_{L,n} \times_1 \bm{V}_{n,1}\ldots\times_K \bm{V}_{n,K} \|_F^2 +\rho \|\bm{V}_{n,k}\bm{V}_{n,k}^\top-\bm{V}_{n,k,t}{\bm{V}_{n,k,t}^\top}\|_F^2
    \end{equation}
    Following a similar process in updating the global factor matrices, the optimization problem above can be expressed as
    \begin{align*}
        &\max_{\bm{V}_{n,k}} \|\bm{V}_{n,k} (\tsR_{L,n})_{(k)}(\bigotimes_{q\neq k} \bm{V}_{n,q}^\top)^\top\|_F^2+2\rho Tr\left[\bm{V}_{n,k}^\top \bm{V}_{n,k,t} {\bm{V}_{n,k,t}^\top} \bm{V}_{n,k}\right]\\
        \iff & \max_{\bm{V}_{n,k}} Tr\left[ \bm{V}_{n,k}^\top\left(\bm{W}_{L,n} \bm{W}_{L,n}^\top + 2\rho \bm{V}_{n,k,t}{\bm{V}_{n,k,t}^\top} \right) \bm{V}_{n,k}\right]
    \end{align*}
    Now denote $\bm{S}=\bm{W}_{L,n} \bm{W}_{L,n}^\top + 2\rho \bm{V}_{n,k,t}{\bm{V}_{n,k,t}^\top}$. If $k\not\in \mathcal K$, there is no further constraint on $\bm{V}_{n,k}$. Therefore, the columns of $\bm{V}_{n,k}$ are the eigenvectors of $\bm{S}$ corresponding to top $l_{n,k}$ eigenvalues.\\
    If $k\in \mathcal K$, the problem becomes
    $$
    \max_{\bm{V}_{n,k}} Tr\left( \bm{V}_{n,k}^{\top} \bm{S} \bm{V}_{n,k}\right)~s.t.~ \bm{V}_{n,k}^{\top} \bm{V}_{n,k}=I, \bm{V}_{n,k}^{\top} \matU_{G,k}=0
    $$
    Therefore, from Lemma \ref{lemma:orthgonal-opt}, if we let $\bm{S}'=(\bm{I}-\matU_{G,k}\matU_{G,k}^\top)\bm{S}(\bm{I}-\matU_{G,k}\matU_{G,k}^\top)$, the columns of $\bm{V}_{n,k}$ are the eigenvectors of $\bm{S}'$ corresponding to top $l_{n,k}$ eigenvalues.    
\end{proof}

\section{Proof of Theorem \ref{thm:convergence}}
In this section, we will prove Theorem \ref{thm:convergence}.  First, we will show that the global factors $\matU_{G,k}$'s converge into stationary points. Then we will show that local factors $\matV_{n,k}$ also converge at stationary points. Some auxiliary lemmas are relegated to Sec. \ref{sec:auxlemma} for clarity.

For notation simplicity, we define
\begin{equation}
\label{eqn:fgidef}
f_{G,n}\left(\matU_{G,1},\matU_{G,2},\cdots,\matU_{G,K}\right)=\norm{\tsY_n-\tsY_n\times_1\matU_{G,1}\matU_{G,1}^\top\times_2\cdots\times_K\matU_{G,K}\matU_{G,K}^\top}_F^2
\end{equation}
and
\begin{equation}
\label{eqn:flidef}
f_{L,n}\left(\matV_{n,1},\matV_{n,2},\cdots,\matV_{n,K}\right)=\norm{\tsY_n-\tsY_n\times_1\matV_{n,1}\matV_{n,1}^\top\times_2\cdots\times_K\matV_{n,K}\matV_{n,K}^\top}_F^2
\end{equation}

Apparently, both $f_{G,n}$ and $f_{L,n}$ are lower bounded by $0$.

The following lemma shows that the global factors $\matU_{G,k}$'s will converge into stationary points.

\begin{lemma}
\label{lm:pugconverge}
The column spaces spanned by $\matU_{G,k}$'s converge into stationary points.
\begin{equation}
\label{eqn:pugconverge}
\sum_{t=1}^T\sum_{k=1}^K\norm{\matU_{G,k,t+1}\matU_{G,k,t+1}^\top-\matU_{G,k,t}\matU_{G,k,t}^\top}_F^2\le \frac{1}{\rho}\sum_{n=1}^Nf_{G,n}\left(\matU_{G,1,1},\matU_{G,2,1},\cdots,\matU_{G,K,1}\right)
\end{equation}
\end{lemma}
Notice that by dividing $TK$ from both sides of \eqref{eqn:pugconverge}, we can immediately prove the first inequality in theorem \ref{thm:convergence}.

\begin{proof}
From proposition \ref{prop:Global_prox}, the update of global factor $j$ is given by
\begin{equation*}
\matU_{G,k,t+1}=\arg\min_{\matU_{G,k}} \sum_{n=1}^N \|\tsR_{G,n,t} - \tsC^{\star}_{G,n} \times_1 \matU_{G,1,t+1}\ldots\times_K \matU_{G,K,t}\|_F^2 +\rho \|\matU_{G,k}\matU_{G,k}^\top-\matU_{G,k,t}{\matU_{G,k,t}^\top}\|_F^2
\end{equation*}
where $\tsC^{\star}_{G,n}=\tsY_n\times_1\matU_{G,1,t+1}\times_2\cdots\times_K\matU_{G,K,t}$. 

By lemma \ref{lemma:min-to-max}, we know 
\begin{equation*}
\sum_{n=1}^N \|\tsR_{G,n} - \tsC^{\star}_{G,n} \times_1 \matU_{G,1}\ldots\times_K \matU_{G,K} \|_F^2=- \|\tsR_{G,n}\times_1 \matU_{G,1}^\top\ldots\times_K \matU_{G,K}^\top \|_F^2+const
\end{equation*}

From the definition of $\tsR_{G,n,t}$, we know that $\tsR_{G,n,t}=\tsY_n-\tsC_{L,n,t}\times_1\matV_{n,1,t}\times_2\cdots\times_K \matV_{n,K,t+1}$. Since $|\setK|\ge 2$ and $\matU_{G,k,t}^\top\matV_{n,k,t}$ for every $k\in \setK$ and every $t$, there exists at least one $m\in \setK$ and $m\neq k$ such that $\matU_{G,m,t}^\top\matV_{n,m,t}=0$ and $\matU_{G,m,t+1}^\top\matV_{n,m,t+1}=0$. Therefore, $\tsR_{G,n}\times_1\matU_{G,1,t+1}^\top\times_2\cdots\times_K \matU_{n,K,t}^\top=\tsY_n\times_1\matU_{G,1,t+1}^\top\times_2\cdots\times_K \matU_{n,K,t}^\top$, 

Then we can apply lemma \ref{lemma:min-to-max} again, and rewrite the maximization problem as the minimization problem,
$$\footnotesize
\begin{aligned}
&\matU_{G,k,t+1}\\
&=\arg\min_{\matU_{G,k}} \sum_{n=1}^N \|\bm{\mathcal Y}_n - \tsY_n \times_1 \matU_{G,1,t+1}\matU_{G,1,t+1}^\top\ldots\times_K \matU_{G,K,t}\matU_{G,K,t}^\top\|_F^2 +\rho \|\matU_{G,k}\matU_{G,k}^\top-\matU_{G,k,t}{\matU_{G,k,t}^\top}\|_F^2\\
&=\arg\min_{\matU_{G,k}}\sum_{n=1}^Nf_{G,n}(\matU_{G,1,t+1},\cdots,\matU_{G,k-1,t+1},\matU_{G,k},\matU_{G,k+1,t},\cdots,\matU_{G,K,t})+\rho \|\matU_{G,k}\matU_{G,k}^\top-\matU_{G,k,t}^\top{\matU_{G,k,t}^\top}\|_F^2
\end{aligned}
$$

From the definition of optimal solution, we know,
\begin{equation}
\label{eqn:dpuinoneiterate}
\begin{aligned}
&\sum_{n=1}^Nf_{G,n}(\matU_{G,1,t+1},\cdots,\matU_{G,k-1,t+1},\matU_{G,k,t+1},\matU_{G,k+1,t},\cdots,\matU_{G,K,t}) \\
&+\rho \|\matU_{G,k,t+1}\matU_{G,k,t+1}^\top-\matU_{G,k,t}{\matU_{G,k,t}^\top}\|_F^2\\
&\le \sum_{n=1}^Nf_{G,n}(\matU_{G,1,t+1},\cdots,\matU_{G,k-1,t+1},\matU_{G,k,t},\matU_{G,k+1,t},\cdots,\matU_{G,K,t})
\end{aligned}    
\end{equation}

Summing both sides for $k$ from $1$ to $K$, then for $t$ from $1$ to $T$, 
we have,
\begin{equation}
\label{eqn:sumofdpuinoneepoch}
\begin{aligned}
&\rho\sum_{k=1}^K\sum_{t=1}^T\|\matU_{G,k,t+1}\matU_{G,k,t+1}^\top-\matU_{G,k,t}{\matU_{G,k,t}^\top}\|_F^2 \\
&\le \sum_{n=1}^Nf_{G,n}(\matU_{G,1,T+1},\cdots,\matU_{G,K,T+1}) -\sum_{n=1}^N f_{G,n}(\matU_{G,1,1},\cdots,\matU_{G,K,1})
\end{aligned}
\end{equation}

Since $f_{G,n}\ge 0$, this completes our proof.
\end{proof}

Now we will analyze the convergence of $\matV_{n,k}$'s.  The following lemma gives an upper bound on the subspace distance of $\matU_{G,k}$'s from one update.
\begin{lemma}
\label{lm:dpubound}
If there exists a constant $B>0$ such that $\norm{\tsY_n}_F\le B$ for all $n$, then we have,
\begin{equation}
\norm{\matU_{G,k,t+1}\matU_{G,k,t+1}^\top-\matU_{G,k,t}{\matU_{G,k,t}^\top}}_F\le \frac{2B\sqrt{N}}{\sqrt{\rho} }
\end{equation}
for all $t=1,2,\cdots,T$ and all $n=1,\cdots,N$.
\end{lemma}

\begin{proof}
Since $\norm{\tsY_n}_F\le B$, by the triangle inequality, we know that $f_{G,n}\le 4B^2$. Applying \eqref{eqn:dpuinoneiterate} and considering the fact that $f_{G,n}\ge 0$, we can prove the desired inequality. 
\end{proof}

We will treat the two cases $k\in \setK$ and $k\notin \setK$ differently. We first consider the case where $k\in \setK$. In Algorithm \ref{alg:perTucker_BCD}, after updating $\matU_{G,k,t+1}$, $\matV_{n,k,t}$ is generally not orthogonal to $\matU_{G,k,t+1}$. However, we are able to show that for any source $n$, there exists an orthonormal matrix $\tmatV_{n,k,t}$ such that $\tmatV_{n,k,t}$ is close to $\matV_{n,k,t}$ and $\tmatV_{n,k,t}^\top\matU_{G,k,t+1}=0$.

\begin{lemma}
\label{lm:tvij}
Under the same conditions as lemma \ref{lm:dpubound}, if additionally the regularization parameter $\rho$ in Algorithm \ref{alg:perTucker_BCD} is not too small $\rho\ge 8B^2N$, then there exists an orthonormal matrix $\tmatV_{n,k,t}$, such that (1) $\norm{\tmatV_{n,k,t}\tmatV_{n,k,t}^\top-\matV_{n,k,t}\matV_{n,k,t}^\top}_F\le 11\norm{\matU_{G,k,t+1}\matU_{G,k,t+1}^\top-\matU_{G,k,t}\matU_{G,k,t}^\top}_F$ and (2) $\tmatV_{n,k,t}^\top\matU_{G,k,t+1}=0$
\end{lemma}
\begin{proof}
We define $\tmatV_{n,k,t}$ as,
$$
\tmatV_{n,k,t}=\left(\matV_{n,k,t}-\matU_{G,k,t+1}\matU_{G,k,t+1}^\top\matV_{n,k,t}\right)\left(\matI-\matV_{n,k,t}^\top\matU_{G,k,t+1}\matU_{G,k,t+1}^\top\matV_{n,k,t}\right)^{-\frac{1}{2}}
$$
It is easy to verify that $\tmatV_{n,k,t}^\top\matU_{G,k,t+1}=0$ and that $\tmatV_{n,k,t}^\top\tmatV_{n,k,t}=\matI$. Now we will prove an upper bound on the norm on the distance between the subspace spanned by $\matV_{n,k,t}$ and by $\tmatV_{n,k,t}$.

\begin{equation*}\scriptsize
\begin{aligned}
&\tmatV_{n,k,t}\tmatV_{n,k,t}^\top-\matV_{n,k,t}\matV_{n,k,t}^\top\\
&=\left(\matV_{n,k,t}-\matU_{G,k,t+1}\matU_{G,k,t+1}^\top\matV_{n,k,t}\right)\left(\matI-\matV_{n,k,t}^\top\matU_{G,k,t+1}\matU_{G,k,t+1}^\top\matV_{n,k,t}\right)^{-1}\left(\matV_{n,k,t}^\top-\matV_{n,k,t}^\top\matU_{G,k,t+1}\matU_{G,k,t+1}^\top\right)\\
&-\left(\matV_{n,k,t}-\matU_{G,k,t+1}\matU_{G,k,t+1}^\top\matV_{n,k,t}\right)\left(\matV_{n,k,t}^\top-\matV_{n,k,t}^\top\matU_{G,k,t+1}\matU_{G,k,t+1}^\top\right)\\
&+\left(\matV_{n,k,t}-\matU_{G,k,t+1}\matU_{G,k,t+1}^\top\matV_{n,k,t}\right)\left(\matV_{n,k,t}^\top-\matV_{n,k,t}^\top\matU_{G,k,t+1}\matU_{G,k,t+1}^\top\right)-\matV_{n,k,t}\matV_{n,k,t}^\top
\end{aligned}
\end{equation*}
Therefore, by triangular inequality,
\begin{equation*}\scriptsize
\begin{aligned}
&\norm{\tmatV_{n,k,t}\tmatV_{n,k,t}^\top-\matV_{n,k,t}\matV_{n,k,t}^\top}_F\\
&\le \norm{\matV_{n,k,t}-\matU_{G,k,t+1}\matU_{G,k,t+1}^\top\matV_{n,k,t}}_2^2\norm{\left(\matI-\matV_{n,k,t}^\top\matU_{G,k,t+1}\matU_{G,k,t+1}^\top\matV_{n,k,t}\right)^{-1}}_F\norm{\matV_{n,k,t}^\top\matU_{G,k,t+1}\matU_{G,k,t+1}^\top\matV_{n,k,t}}_F\\
&+\norm{\matU_{G,k,t+1}\matU_{G,k,t+1}^\top\matV_{n,k,t}\matV_{n,k,t}^\top}_F+\norm{\matV_{n,k,t}\matV_{n,k,t}^\top\matU_{G,k,t+1}\matU_{G,k,t+1}^\top}_F\\
&+\norm{\matU_{G,k,t+1}\matU_{G,k,t+1}^\top\matV_{n,k,t}\matV_{n,k,t}^\top\matU_{G,k,t+1}\matU_{G,k,t+1}^\top}_F\\
&\le 4\norm{\left(\matI-\matV_{n,k,t}^\top\matU_{G,k,t+1}\matU_{G,k,t+1}^\top\matV_{n,k,t}\right)^{-1}}_F\norm{\matV_{n,k,t}^\top\matU_{G,k,t+1}}_F^2+2\norm{\matV_{n,k,t}^\top\matU_{G,k,t+1}}_F+\norm{\matV_{n,k,t}^\top\matU_{G,k,t+1}}_F^2
\end{aligned}
\end{equation*}
where the last inequality comes from the fact that 
$$
\begin{aligned}
&\norm{\matV_{n,k,t}-\matU_{G,k,t+1}\matU_{G,k,t+1}^\top\matV_{n,k,t}}_2\\
&\le \norm{\matV_{n,k,t}}+\norm{\matU_{G,k,t+1}\matU_{G,k,t+1}^\top\matV_{n,k,t}}_2\\
&\le 1+\norm{\left(\matU_{G,k,t+1}\matU_{G,k,t+1}^\top-\matU_{G,k,t}\matU_{G,k,t}^\top\right)\matV_{n,k,t}}_2\\
&\le 1+ \frac{2B\sqrt{N}}{\sqrt{\rho}}\le 2
\end{aligned}
$$

Since $\matU_{G,k,t}$ and $\matV_{n,k,t}$ are orthogonal $\matU_{G,k,t}^\top\matV_{n,k,t}=0$, we have,
\begin{equation*}
\begin{aligned}
&\norm{\matU_{G,k,t+1}^\top\matV_{n,k,t}}_F=
\norm{\matU_{G,k,t+1}\matU_{G,k,t+1}^\top\matV_{n,k,t}}_F\\
&=\norm{\matU_{G,k,t+1}\matU_{G,k,t+1}^\top\matV_{n,k,t}-\matU_{G,k,t}\matU_{G,k,t}^\top\matV_{n,k,t}}_F\\
&\le \norm{\matU_{G,k,t+1}\matU_{G,k,t+1}^\top-\matU_{G,k,t}\matU_{G,k,t}^\top}_F
\end{aligned}
\end{equation*}

Thus, when $\scriptstyle\norm{\matU_{G,k,t+1}\matU_{G,k,t+1}^\top-\matU_{G,k,t}\matU_{G,k,t}^\top}_F\le \frac{\sqrt{2}}{2}$, we know $\scriptstyle\small\norm{\matV_{n,k,t}^\top\matU_{G,k,t+1}\matU_{G,k,t+1}^\top\matV_{n,k,t}}_F\le \frac{1}{2}$. As a result, $\scriptstyle\norm{\left(\matI-\matV_{n,k,t}^\top\matU_{G,k,t+1}\matU_{G,k,t+1}^\top\matV_{n,k,t}\right)^{-1}}_F\le 2$.

Combining this with the fact that $\norm{\matU_{G,k,t+1}^\top\matV_{n,k,t}}_F\le 1$, we have,
$$
\norm{\tmatV_{n,k,t}\tmatV_{n,k,t}^\top-\matV_{n,k,t}\matV_{n,k,t}^\top}_F\le 11\norm{\matU_{G,k,t+1}\matU_{G,k,t+1}^\top-\matU_{G,k,t}\matU_{G,k,t}^\top}_F
$$
This completes the proof.
\end{proof}

Finally, we will prove the second equation in Theorem \ref{thm:convergence}.

\begin{proof}
We start by analyzing the update rule of $\matV_{n,k}$. In Proposition \ref{prop:Local_prox}, we consider the update rule for two cases depending on whether $k\in \setK$. We analyze them separately.

\underline{Case I, $k\notin\setK$}
In this case, $\matV_{n,k,t+1}$ should be the optimal solution to,
\begin{equation*}
\bm{V}_{n,k,t+1}=\arg\min_{\bm{V}_{n,k}} \|\tsR_{L,n,t} - \tsC^{\star}_{L,n} \times_1 \bm{V}_{n,1,t+1}\ldots\times_K \bm{V}_{n,K,t} \|_F^2 +\rho \|\bm{V}_{n,k}\bm{V}_{n,k}-\bm{V}_{n,k,t}{\bm{V}_{n,k,t}^\top}\|_F^2
\end{equation*}
where $\tsC^{\star}_{L,n}$ is optimized according to the factors $\matV_{n,1,t+1},\cdots,\matV_{n,K,t}$. Thus we have,
\begin{align*}
&\arg\min_{\bm{V}_{n,k}} \|\tsR_{L,n,t} - \tsC^{\star}_{L,n} \times_1 \bm{V}_{n,1,t+1}\ldots\times_K \bm{V}_{n,K,t} \|_F^2 +\rho \|\bm{V}_{n,k}\bm{V}_{n,k}^\top-\bm{V}_{n,k,t}{\bm{V}_{n,k,t}^\top}\|_F^2\\
&=\arg\min_{\bm{V}_{n,k}} -\|\tsR_{L,n,t}\times_1 \bm{V}_{n,1,t+1}^\top\ldots\times_K \bm{V}_{n,K,t}^\top \|_F^2 +\rho \|\bm{V}_{n,k}\bm{V}_{n,k}^\top-\bm{V}_{n,k,t}{\bm{V}_{n,k,t}^\top}\|_F^2
\end{align*}

Since $\abs{\setK}\ge 2$, there exists at least one $m\in \setK$ and $m\neq k$ such that $\matU_{G,m,t}^\top\matV_{n,m,t}=0$ and $\matU_{G,m,t+1}^\top\matV_{n,m,t+1}=0$. We thus have 
$$
\|\tsR_{L,n,t}\times_1 \bm{V}_{n,1,t+1}^\top\ldots \times_K \bm{V}_{n,K,t}^\top\|_F^2=\|\bm{\mathcal Y}_n\times_1 \bm{V}_{n,1,t+1}^\top\ldots \times_K \bm{V}_{n,K,t}^\top\|_F^2
$$
Therefore, we can use Lemma \ref{lemma:min-to-max} again to derive

{\footnotesize \begin{align*}
\matV_{n,k,t+1}&=\arg\min_{\matV_{n,k}}\norm{\tsY_n-\tsC_{L,n,t}\times_1\matV_{n,1,t+1}\cdots\times_K\matV_{n,K,t}}_F^2+\rho\norm{\matV_{n,k}\matV_{n,k}^\top-\matV_{n,k,t}\matV_{n,k,t}^\top}_F^2\\
    &=\arg\min_{\matV_{n,k}}f_{L,n}(\matV_{n,1,t+1},\cdots,\matV_{n,k-1,t+1},\matV_{n,k},\matV_{n,k+1,t},\cdots,\matV_{n,K,t})+\rho\norm{\matV_{n,k}\matV_{n,k}^\top-\matV_{n,k,t}\matV_{n,k,t}^\top}_F^2
\end{align*}}%

Since $\matV_{n,k,t}$ is a feasible solution, we know

\begin{equation}
\label{eqn:jnotinsetk}
\begin{aligned}
&f_{L,n}(\matV_{n,1,t+1},\cdots,\matV_{n,k,t+1},\matV_{n,k+1,t},\cdots,\matV_{n,K,t}) +\rho \|\bm{V}_{n,k,t+1}\bm{V}_{n,k,t+1}^\top-\bm{V}_{n,k,t}{\bm{V}_{n,k,t}^\top}\|_F^2\\
&\le f_{L,n}(\matV_{n,1,t+1},\cdots,\matV_{n,k-1,t+1},\matV_{n,k,t},\cdots,\matV_{n,K,t})
\end{aligned}    
\end{equation}

\underline{Case II, $k\in \setK$}
By Proposition \ref{prop:Local_prox}, the update rule of $\matV_{n,k}$ is this case is given by,
\begin{equation*}
\bm{V}_{n,k,t+1}=\arg\min_{\bm{V}_{n,k}\perp \matU_{G,k,t+1}} \|\tsR_{L,n,t} - \tsC^{\star}_{L,n} \times_1 \bm{V}_{n,1,t+1}\ldots\times_K \bm{V}_{n,K,t} \|_F^2 +\rho \|\bm{V}_{n,k}\bm{V}_{n,k}^\top-\bm{V}_{n,k,t}{\bm{V}_{n,k,t}^\top}\|_F^2
\end{equation*}
where $\tsC^{\star}_{L,n}$ is optimized according to the factors $\matV_{n,1,t+1},\cdots,\matV_{n,K,t}$. Thus we have,
\begin{align*}\small
\|\tsR_{L,n,t} - \tsC^{\star}_{L,n} \times_1 \bm{V}_{n,1,t+1}\ldots\times_K \bm{V}_{n,K,t}\|_F^2 = \|\tsR_{L,n,t}\|_F^2 - \|\tsR_{L,n,t}\times_1 \bm{V}_{n,1,t+1}^\top\ldots \times_K \bm{V}_{n,K,t}^\top\|_F^2 
\end{align*}

Therefore, we can rewrite the equivalent formulation to $\matV_{n,k}$ as,
{\small
    \begin{equation*}
\bm{V}_{n,k,t+1}=\underset{\bm{V}_{n,k}\perp \matU_{G,k,t+1}}{\mathrm{argmin}} f_{L,n}(\matV_{n,1,t+1},\cdots,\matV_{n,k},\matV_{n,k+1,t},\cdots,\matV_{n,K,t}) +\rho \|\bm{V}_{n,k}\bm{V}_{n,k}^\top-\bm{V}_{n,k,t}{\bm{V}_{n,k,t}^\top}\|_F^2
\end{equation*}}%

The $\tmatV_{n,k,t}$ defined in lemma \ref{lm:tvij} is an feasible solution, therefore we have,
$$
\begin{aligned}
&f_{L,n}(\matV_{n,1,t+1},\cdots,\matV_{n,k-1,t+1},\matV_{n,k,t+1},\matV_{n,k+1,t},\cdots,\matV_{n,K,t}) +\rho \|\bm{V}_{n,k,t+1}\bm{V}_{n,k,t+1}^\top-\bm{V}_{n,k,t}{\bm{V}_{n,k,t}^\top}\|_F^2\\
&\le f_{L,n}(\matV_{n,1,t+1},\cdots,\matV_{n,k-1,t+1},\tmatV_{n,k,t},\matV_{n,k+1,t},\cdots,\matV_{n,K,t}) +\rho \|\tmatV_{n,k,t}\tmatV_{n,k,t}^\top-\bm{V}_{n,k,t}{\bm{V}_{n,k,t}^\top}\|_F^2
\end{aligned}
$$
We can bound the difference between $f_{L,n}(\matV_{n,1,t+1},\cdots,\matV_{n,k-1,t+1},\tmatV_{n,k,t},\matV_{n,k+1,t},\cdots,\matV_{n,K,t})$ and $f_{L,n}(\matV_{n,1,t+1},\cdots,\matV_{n,k-1,t+1},\matV_{n,k,t},\matV_{n,k+1,t},\cdots,\matV_{n,K,t})$ as,
$$\scriptsize
\begin{aligned}
&f_{L,n}(\matV_{n,1,t+1},\cdots,\matV_{n,k-1,t+1},\tmatV_{n,k,t},\matV_{n,k+1,t},\cdots,\matV_{n,K,t})\\
&\le \Big(\norm{\tsY_n-\tsY_n\times_1\cdots\times_k\matV_{n,k,t}\matV_{n,k,t}^\top\times_K\matV_{n,K,t}}_F+\norm{\tsY_n\times_1\cdots\times_k\left(\matV_{n,k,t}\matV_{n,k,t}^\top-\tmatV_{n,k,t}\tmatV_{n,k,t}^\top\right)\times_K\matV_{n,K,t}}_F\Big)^2\\
&\le \norm{\tsY_n-\tsY_n\times_1\cdots\times_k\matV_{n,k,t}\matV_{n,k,t}^\top\times_K\matV_{n,K,t}}_F^2\\
&+2\norm{\tsY_n-\tsY_n\times_1\cdots\times_k\matV_{n,k,t}\matV_{n,k,t}^\top\times_K\matV_{n,K,t}}_F\norm{\tsY_n\times_1\cdots\times_k\left(\matV_{n,k,t}\matV_{n,k,t}^\top-\tmatV_{n,k,t}\tmatV_{n,k,t}^\top\right)\times_K\matV_{n,K,t}}_F\\
&+\norm{\tsY_n\times_1\cdots\times_k\left(\matV_{n,k,t}\matV_{n,k,t}^\top-\tmatV_{n,k,t}\tmatV_{n,k,t}^\top\right)\times_K\matV_{n,K,t}}_F^2\\
\end{aligned}
$$
The first term above is just $f_{L,n}(\matV_{n,1,t+1},\cdots,\matV_{n,k-1,t+1},\matV_{n,k,t},\matV_{n,k+1,t},\cdots,\matV_{n,K,t})$. By Lemma \ref{lm:tuckerprodupper}, the second term is upper bounded by $\scriptstyle 2\norm{\tsY_n}_F\norm{\tsY_n}_F\norm{\matV_{n,k,t}\matV_{n,k,t}^\top-\tmatV_{n,k,t}\tmatV_{n,k,t}^\top}_2$, and the third term is bounded by $\scriptstyle\norm{\matV_{n,k,t}\matV_{n,k,t}^\top-\tmatV_{n,k,t}\tmatV_{n,k,t}^\top}_2^2\norm{\tsY_n}_F^2$.

Therefore, we have
$$
\begin{aligned}
&f_{L,n}(\matV_{n,1,t+1},\cdots,\matV_{n,k-1,t+1},\tmatV_{n,k,t},\matV_{n,k+1,t},\cdots,\matV_{n,K,t})\\
&\le f_{L,n}(\matV_{n,1,t+1},\cdots,\matV_{n,k-1,t+1},\matV_{n,k,t},\cdots,\matV_{n,K,t}) + \norm{\tmatV_{n,k,t}\tmatV_{n,k,t}^\top-\matV_{n,k,t}\matV_{n,k,t}^\top}_24B^2
\end{aligned}
$$

Therefore, we have
\begin{equation}
\label{eqn:jinsetk}
\begin{aligned}
&f_{L,n}(\matV_{n,1,t+1},\cdots ,\matV_{n,k,t+1},\matV_{n,k+1,t},\cdots,\matV_{n,K,t})  +\rho \|\bm{V}_{n,k,t+1}\bm{V}_{n,k,t+1}^\top-\bm{V}_{n,k,t}{\bm{V}_{n,k,t}^\top}\|_F^2\\
&- f_{L,n}(\matV_{n,1,t+1},\cdots,\matV_{n,k-1,t+1},\matV_{n,k,t},\matV_{n,k+1,t},\cdots,\matV_{n,K,t})\\
&\le 4B^2\norm{\tmatV_{n,k,t}\tmatV_{n,k,t}^\top-\matV_{n,k,t}\matV_{n,k,t}^\top}_2+\rho \norm{\tmatV_{n,k,t}\tmatV_{n,k,t}^\top-\matV_{n,k,t}\matV_{n,k,t}^\top}_F^2\\
&\le C_3\norm{\matU_{G,k,t+1}\matU_{G,k,t+1}^\top-\matU_{G,k,t}\matU_{G,k,t}^\top}_F
\end{aligned}
\end{equation}
where $C_3$ is a constant defined as,
$$
C_3=11B^2+242B\sqrt{N\rho}
$$
In the last inequality of \eqref{eqn:jinsetk}, we applied Lemma \ref{lm:dpubound} and Lemma \ref{lm:tvij}.

Combining equation \eqref{eqn:jnotinsetk} in Case I and equation \eqref{eqn:jinsetk} in Case II, we know that,
\begin{align*}
&  f_{L,n}(\matV_{n,1,t+1},\cdots,\matV_{n,K,t+1}) -f_{L,n}(\matV_{n,1,t},\cdots,\matV_{n,K,t})\\
&\le \sum_{k=1}^K\rho \|\bm{V}_{n,k,t+1}\bm{V}_{n,k,t+1}^\top-\bm{V}_{n,k,t}{\bm{V}_{n,k,t}^\top}\|_F^2+\sum_{k\in\setK}C_3\|\matU_{G,k,t+1}\matU_{G,k,t+1}^\top-\bm{U}_{n,k,t}{\bm{U}_{n,k,t}^\top}\|_F\\
&\le \sum_{k=1}^K\rho \|\bm{V}_{n,k,t+1}\bm{V}_{n,k,t+1}^\top-\bm{V}_{n,k,t}{\bm{V}_{n,k,t}^\top}\|_F^2+\sum_{k=1}^KC_3\|\matU_{G,k,t+1}\matU_{G,k,t+1}^\top-\bm{U}_{n,k,t}{\bm{U}_{n,k,t}^\top}\|_F
\end{align*}
Summing both sides for $t$ from $1$ to $T$ and for $k$ from $1$ to $K$, we have
{\footnotesize \begin{align*}
& \sum_{k=1}^K\sum_{t=1}^T\rho\|\bm{V}_{n,k,t+1}\bm{V}_{n,k,t+1}^\top-\bm{V}_{n,k,t}{\bm{V}_{n,k,t}^\top}\|_F^2\le C_3\sum_{k=1}^K\sum_{t=1}^T \|\matU_{G,k,t+1}\matU_{G,k,t+1}^\top-\bm{U}_{n,k,t}{\bm{U}_{n,k,t}^\top}\|_F + f_{L,n,0}\\
\end{align*}}%

Considering lemma \ref{lm:pugconverge} and the fact that 
$$
\sum_{k=1}^K\sum_{t=1}^T \|\matU_{G,k,t+1}\matU_{G,k,t+1}^\top-\bm{V}_{n,k,t}{\bm{V}_{n,k,t}^\top}\|_F\le \sqrt{TK}\sqrt{\sum_{k=1}^K\sum_{t=1}^T \|\matU_{G,k,t+1}\matU_{G,k,t+1}^\top-\bm{V}_{n,k,t}{\bm{V}_{n,k,t}^\top}\|_F^2}
$$
we can devide both sides by $T$ and conclude that,
\begin{align*}
& \min_{t\in\{1,\cdots,T\}} \sum_{k=1}^K\|\bm{V}_{n,k,t+1}\bm{V}_{n,k,t+1}^\top-\bm{V}_{n,k,t}{\bm{V}_{n,k,t}^\top}\|_F^2=O\left(\frac{1}{\sqrt{T}}\right)
\end{align*}
This completes our proof.
\end{proof}

\section{Simulation Study about the Convergence Rate}\label{app:ConvergenceSimulation}
In this simulation, we generate the data according to \eqref{eq: perTuckerModel} but without noise for better visualization. We assume that there are 5 sources of data, each has 10 samples, and each sample has dimension $50\times 50\times 50$. Therefore, the data is $\tsY_1,\ldots,\tsY_5$ each with dimension $50\times 50\times 50\times10$.

For data generation, we use the standard normal distribution to generate $\matU_{G,1},\matU_{G,2},\matU_{G,3}$ with dimension $50\times5$ and use SVD to construct an orthonormal basis to make the matrices orthonormal. Then we use the standard normal distribution to generate $\tsC_{G,1},\ldots,\tsC_{G,5}$ with dimension $5\times5\times5\times10$. For $n=1,\ldots,5$, the global component is constructed by $\tsY_{G,n}=\tsC_{G,n}\times_1\matU_{G,1} \times_2\matU_{G,2} \times_3\matU_{G,3}$. When generating the local factor matrices, we assume $\mathcal K=\{1,2\}$. Then for each $n\in\{1,\ldots,5\}$, if $k\in\mathcal K$, $\matV_{n,k}$ are first generated by a standard normal distribution with dimension $50\times 5$. Afterwards, we use $(\bm{I}-\matU_{G,k}(\matU_{G,k}^\top \matU_{G,k})^{-1}\matU_{G,k})\matV_{n,k}$ to project the local factor matrix onto the orthogonal space of the corresponding global factor matrix. Finally, we use SVD to construct the orthonormal basis. If $k\not\in\mathcal K$, we use the standard normal distribution to generate $\matV_{n,k}$ and use SVD to construct an orthonormal basis. Similarly, we generate the local core tensors $\tsC_{L,1},\ldots,\tsC_{L,5}$ from a standard normal distribution with dimension $5\times5\times5\times10$. The local components for $n=1,\ldots,5$ are constructed by $\tsY_{L,n}=\tsC_{L,n}\times_1\matV_{n,1} \times_2\matV_{n,2} \times_3\matV_{n,3}$.

When performing \name, we use the dimension of true global and local factor matrices. Monitoring statistics along the iteration axis are mean subspace error for global factor matrices and local factor matrices for each client, that is, $\frac13 \sum_{k=1}^3 \|\matU_{G,k,t}\matU_{G,k,t}^\top-\matU_{G,k}\matU_{G,k}^\top\|_F^2$ for global monitoring and $\frac13 \sum_{k=1}^3 \|\matV_{n,k,t}\matV_{n,k,t}^\top-\matV_{n,k}\matV_{n,k}^\top\|_F^2$ for source $n$ local monitoring.

\section{Discussion of Classification Rule in Sec. \ref{subsubsec:Application_Classification}}\label{app:ClassificationDiscussion}
From Proposition \ref{prop:Core}, when a new piece of data arrives, we can first derive the optimal global core tensor $\tsC^{\star}_{G}$, and for each class $n$, we can derive the optimal local tensor $\tsC^{\star}_{L,n}$. Then the decision rule by minimizing the reconstruction error from all the classes is given by
\begin{equation}\label{eq:classification}
    \hat{n} = \arg\min_n \norm{\bm{\mathcal Y}^{\text{new}} - \bm{\mathcal C}^{\star}_{G}\times_1 \hat{\matU}_{G,1}\ldots \times_K \hat{\matU}_{G,K}- \bm{\mathcal C}^{\star}_{L,n}\times_1 \hat{\bm{V}}_{n,1}  \ldots \times_K \hat{\bm{V}}_{n,K} }_F^2.
\end{equation}

We then show that the decision rule \eqref{eq:classification} is equivalent to the decision rule \eqref{eqn:classifystatistics}.
    By reformulating Equation \eqref{eq:classification}, we have
    \begin{align*}
        &\|\bm{\mathcal Y}^{\text{new}} - \tsC^{\star}_{G}\times_1 \hat{\matU}_{G,1}\ldots \times_K \hat{\matU}_{G,K}- \tsC^{\star}_{L,n}\times_1 \hat{\bm{V}}_{n,1}  \ldots \times_K \hat{\bm{V}}_{n,K} \|_F^2.\\
        =& \|\bm{\mathcal Y}^{\text{new}}\|_F^2 + \|\tsC^{\star}_{G}\times_1 \hat{\matU}_{G,1}\ldots \times_K \hat{\matU}_{G,K}\|_F^2 + \|\tsC^{\star}_{L,n}\times_1 \hat{\bm{V}}_{n,1}  \ldots \times_K \hat{\bm{V}}_{n,K}\|_F^2\\
        &-2\langle \bm{\mathcal Y}^{\text{new}}, \tsC^{\star}_{G}\times_1 \hat{\matU}_{G,1}\ldots \times_K \hat{\matU}_{G,K} \rangle -2\langle \bm{\mathcal Y}^{\text{new}}, \tsC^{\star}_{L,n}\times_1 \hat{\bm{V}}_{n,1}  \ldots \times_K \hat{\bm{V}}_{n,K} \rangle \\
        &-2\langle \tsC^{\star}_{G}\times_1 \hat{\matU}_{G,1}\ldots \times_K \hat{\matU}_{G,K}, \tsC^{\star}_{L,n}\times_1 \hat{\bm{V}}_{n,1}  \ldots \times_K \hat{\bm{V}}_{n,K}\rangle
    \end{align*}
    The first, second, and fourth term is irrelevant to $n$ and thus can be excluded. The last term is $0$ because the global and local components are orthogonal. The third and fifth terms can be reformulated as
    \begin{align*}
        &\|\tsC^{\star}_{L,n}\times_1 \hat{\bm{V}}_{n,1}  \ldots \times_K \hat{\bm{V}}_{n,K}\|_F^2 -2 \langle \bm{\mathcal Y}, \tsC^{\star}_{L,n}\times_1 \hat{\bm{V}}_{n,1}  \ldots \times_K \hat{\bm{V}}_{n,K} \rangle\\
        =& \|\tsC^{\star}_{L,n}\|_F^2 -2 \langle \bm{\mathcal Y}^{\text{new}}\times_1 \hat{\bm{V}}_{n,1}^\top  \ldots \times_K \hat{\bm{V}}_{n,K}^\top, \tsC^{\star}_{L,n} \rangle\\
        =& -\|\tsC^{\star}_{L,n}\|_F^2
    \end{align*}
    Therefore, 
    $$
    \arg\min_n \|\bm{\mathcal Y}^{\text{new}} - \tsC^{\star}_{G}\times_1 \hat{\matU}_{G,1}\ldots \times_K \hat{\matU}_{G,K}- \tsC^{\star}_{L,n}\times_1 \hat{\bm{V}}_{n,1}  \ldots \times_K \hat{\bm{V}}_{n,K} \|_F^2
    =\arg\max_n \|\tsC^{\star}_{L,n}\|_F^2
    $$

\section{Auxiliary Lemma}
\label{sec:auxlemma}
In this section, we will present several auxiliary lemmas useful for theoretical analysis.

\begin{lemma}
\label{lm:abupper}

For two matrices $A\in \mathbb{R}^{m\times n}$, and $B\in \mathbb{R}^{n\times p}$, we have,
$$
\norm{AB}_F\le \norm{A}_F\norm{B}_2
$$
\end{lemma}
This lemma is the same as Proposition B.4 in \cite{ruoyuconvergence}. Thus we omit the proof here to avoid duplication.

\begin{lemma}
\label{lm:tuckerprodupper}
For a core tensor $\tsC$ and $K$ factor matrices $\{\matU_k\}_{k=1}^K$, the following holds,
$$
\norm{\tsC\times_1\matU_1\times_2\cdots\times_K\matU_K}_F\le \norm{\tsC}_F\prod_{k=1}^K\norm{\matU_k}_2
$$
\end{lemma}
\begin{proof}
The proof is straightforward. We use $\tsY$ to denote $\tsY=\tsC\times_1\matU_1\times_2\cdots\times_K\matU_K$. Then $\tsY_{(1)}=\matU_1\tsC_{(1)}\left(\matU_K\bigotimes\matU_{K-1}\cdots\bigotimes\matU_2\right)$. Since $\norm{\tsY}_F=\norm{\tsY_{(1)}}_F$, we have,
\begin{align*}
&\norm{\tsY}_F=\norm{\tsY_{(1)}}_F=\norm{\matU_1\tsC_{(1)}\left(\matU_K\bigotimes\matU_{K-1}\cdots\bigotimes\matU_2\right)^\top}_F\\
&\le \norm{\matU_1}_2\norm{\tsC_{(1)}}_F\norm{\left(\matU_K\bigotimes\matU_{K-1}\cdots\bigotimes\matU_2\right)}_2\\
&\le \norm{\tsC_{(1)}}_F\prod_{k=1}^K\norm{\matU_k}_2=\norm{\tsC}_F\prod_{k=1}^K\norm{\matU_k}_2
\end{align*}
where we used Lemma \ref{lm:abupper} in the first inequality, and the fact that $\norm{\matU_K\bigotimes\matU_{K-1}\cdots\bigotimes\matU_2}_2\linebreak\le \norm{\matU_K}_2\cdots\norm{\matU_2}_2$ in the second inequality.
\end{proof}

\end{document}